\documentclass[12pt]{article}
\usepackage{enumerate}
\usepackage[OT1]{fontenc}
\usepackage{amsmath,amssymb}
\usepackage{natbib}
\usepackage[usenames]{color}

\usepackage{smile}
\usepackage{natbib}
\usepackage{multirow}
\usepackage[breaklinks,colorlinks,
            linkcolor=red,
            anchorcolor=blue,
            citecolor=blue
            ]{hyperref}

\def\tr{\mathop{\text{tr}}\kern.2ex}

\long\def\comment#1{}

\def\tr{\mathop{\text{Tr}}}

\def\cS{{\mathcal{S}}}

\def\cD{{\mathcal{D}}}
\def\cP{{\mathcal{P}}}

\def\cN{{\mathcal{N}}}

\def\cO{{\mathcal{O}}}
\def\tr{{\text{Tr}}}

\newcommand{\bel}{\begin{eqnarray}\label}
\newcommand{\eel}{\end{eqnarray}}
\newcommand{\bes}{\begin{eqnarray*}}
\newcommand{\ees}{\end{eqnarray*}}

\newcommand{\red}{\color{red}}

\newcommand{\la}{\langle}
\newcommand{\ra}{\rangle}

\let\emptyset\varnothing

\usepackage{mathrsfs}

\usepackage{fullpage}

\def\sep{\;\big|\;}

\def \algo1 {OVI-SE }

\usepackage{hyperref}
\usepackage{breakcites}
\def\##1\#{\begin{align}#1\end{align}}
\def\$#1\${\begin{align*}#1\end{align*}}
\usepackage{enumitem}
\title{\LARGE Can Reinforcement Learning Find Stackelberg-Nash Equilibria in General-Sum Markov Games with Myopic Followers?}
\author{Han Zhong\thanks{Peking University. Email: \texttt{hanzhong@stu.pku.edu.cn}} \qquad  Zhuoran Yang\thanks{Princeton University. Email: \texttt{zy6@princeton.edu}} \qquad Zhaoran Wang\thanks{Northwestern University. Email: \texttt{zhaoranwang@gmail.com}}\qquad Michael I. Jordan\thanks{UC Berkeley. Email: \texttt{jordan@cs.berkeley.edu}}}
\date{}
\begin{document}
\maketitle


\begin{abstract}
    We study multi-player general-sum Markov games with one of the players designated as the leader and the other players regarded as followers. 
    In particular, we focus on the class of games where the followers are myopic, i.e., they aim to maximize their instantaneous rewards. 
    For such a game, our goal is to find a Stackelberg-Nash equilibrium (SNE), which is a policy  pair $(\pi^*, \nu^*)$ such that (i) $\pi^*$ is the optimal policy for the leader when the followers always play their best response, and (ii) $\nu^*$ is the best response policy of the followers, which is a Nash equilibrium of the followers' game induced by $\pi^*$. We develop sample-efficient reinforcement learning (RL) algorithms for solving for an SNE in both online and offline settings. Our algorithms are optimistic and pessimistic variants of least-squares value iteration, and they are readily able to incorporate function approximation tools in the setting of large state spaces. Furthermore, for the case with linear function approximation,  we prove that our algorithms achieve  sublinear regret and suboptimality under online and offline setups respectively. 
    To the best of our knowledge, we establish the first provably efficient RL algorithms for solving for SNEs in general-sum Markov games with myopic followers. 
\end{abstract}

\section{Introduction}

Reinforcement learning (RL) has achieved striking empirical successes in solving real-world sequential decision-making problems \citep{ mnih2015human,duan2016benchmarking,silver2016mastering, silver2017mastering, silver2018general, agostinelli2019solving, akkaya2019solving}. Motivated by these successes, multi-agent extensions of RL algorithms have recently become popular in  decision-making problems involving multiple interacting agents \citep{busoniu2008comprehensive, hernandez2018multiagent,  hernandez2019survey, oroojlooyjadid2019review, zhang2019multi}.
Multi-agent RL is often modeled as a Markov game \citep{littman1994markov} where, at each time step, given the state of the environment, each player (agent) takes an action simultaneously, observes her own immediate reward, and the environment evolves into a next state. 
Here both the reward of each player and the state transition depend  on the  actions of all players.  
From the perspective of each player, her goal is to find a policy that maximizes her expected total reward in the presence of other agents. 

In Markov games, depending on the structure of the reward functions, the relationship among the players can be either  collaborative, where each player has the same reward function, or   competitive, where the sum of the reward function is  equal to zero, or mixed, which corresponds to a general-sum game. 
While most of the existing theoretical results focus on the  collaborative or two-player competitive settings, 
the mixed setting is oftentimes more pertinent to real-world multi-agent applications. 

Moreover, in addition to having diverse reward functions, the players might also have asymmetric roles in the Markov game---the players might be divided into  leaders and followers, where 
the leaders' joint policy determines a general-sum game for the followers.  
Games with such a leader-follower structure is popular in applications such as mechanism design \citep{conitzer2002complexity,roughgarden2004stackelberg, garg2005design,  kang2014incentive}, security games \citep{tambe2011security, korzhyk2011stackelberg}, incentive design \citep{zheng1984stackelberg, ratliff2014social, chen2016caching,  ratliff2020adaptive}, and 
model-based RL \citep{rajeswaran2020game}.
Consider a simplified economic system that consists of a government and a group of companies, where the companies purchase or sell goods, and the government collects taxes from transactions. 
Such a problem can be viewed as a multi-player general-sum game, where the government serves as the leader 
and the companies are followers \citep{zheng2020ai}.
In particular, 
when the government sets a tax rate, the companies form a general-sum game among themselves with a reward function that depends on the tax rate. 
Each company aims to maximize its own revenue, and ideally they jointly achieve a Nash equilibrium (NE) of the induced game. 
The goal of the government might be to achieve maximum social welfare, which can  be measured via certain fairness metrics associated with the revenues of the companies. 

In multi-player Markov games with such a leader-follower structure, 
the desired solution concept is the Stackelberg-Nash equilibrium (SNE)
\citep{bacsar1998dynamic}.
In the setting where there is a single leader,  
SNE corresponds to a pair of  leader's  policy $\pi^*$  and followers' joint policy $\nu^*$ that satisfies the following two properties: (i) when the leader  adopts $\pi^*$, $\nu^*$ is the best-response policy of the followers, i.e., $\nu^*$ is a Nash equilibrium of the followers' subgame induced by $\pi^*$;  and (ii) $\pi^*$ is the 
optimal policy of the leader assuming the followers always adopt the best response.  

We are interested in finding an SNE in a multi-player Markov game when 
the reward functions and Markov transition kernel are unknown. In particular, we focus on the setting with a single leader and multiple \emph{myopic} followers. That is, the followers at any step of the game do not take into account the future rewards, but only the rewards in the  current step. The formal definition of myopic followers is given in \S \ref{sec:def:myopic:follower}. This setting is a natural formalization of many real-world problems such as marketing and supply chain management. For example, in a market, the leader is an established firm and the followers are entrants. The entrants are not sure whether the firm is going to exist in the future, so they might just want to maximize instantaneous rewards. 
 See \citet{li2017review,kanska2021dynamic} and references therein for more examples. 
For such a game, we are interested in the following question: 
 \begin{center}
 	Can we develop sample-efficient reinforcement learning methods that provably find Stackelberg-Nash equilibria in general-sum Markov games with myopic followers?
 \end{center}
To this end, we consider both online and offline RL settings, where in the former, we learn the SNE in a trial-and-error fashion by interacting with the environment and generating data, and in the latter, we learn the SNE from a  given dataset that is collected a priori. 
For the online setting,  as the transition model is unknown, to achieve sample efficiency 
the equilibrium-finding algorithm needs to take the exploration-exploitation tradeoff into consideration.
Although a similar challenge has been studied in zero-sum Markov game, it seems unclear how to incorporate  popular  exploration mechanisms  such as optimism in the face of uncertainty \citep{sutton2018reinforcement} into the SNE solution concept.
Meanwhile, in the offline setting, as the RL agent has no  control of data collection, 
it is necessary to design an RL algorithm with theoretical guarantees for an arbitrary dataset that might not be sufficiently explorative.  
 
 {\noindent \bf Our contributions} \quad  Our contributions are three-fold. 
 First, for the episodic general-sum Markov game with myopic followers, under  the online and offline  settings respectively, we propose  optimistic and  pessimistic variants of the least-squares value iteration (LSVI) algorithm. 
 In particular, in a version of LSVI, we estimate the optimal action-value function of the leader via least-squares regression  and construct an estimate of the SNE by solving for the SNE of the multi-matrix game for each state, whose payoff matrices are given  by the leader's estimated action-value function and the followers' reward functions. 
 Moreover, we add a UCB exploration bonus to the least-squares solution to achieve optimism in the online setting. 
 In the offline setting, pessimism is achieved by subtracting a penalty function constructed using the offline data which is equal to the negative bonus function. 
 These algorithms are readily able to incorporate function approximators and we showcase a version with linear function approximation.  
 Second, in the online setting, we prove that our optimistic LSVI algorithm achieves a sublinear $\tilde \cO( H^2 \sqrt{d^3   K })$ regret, where $K$ is the number of episodes,  $H$ is the horizon, $d$ is the dimension of the feature mapping, and $\tilde \cO(\cdot )$ omits logarithmic terms. 
 Finally, in the offline setting, we establish an upper bound on the suboptimality of the proposed algorithm for an arbitrary dataset with $K$ trajectories. 
 Our upper bound yields a sublinear $\tilde \cO( H^2 \sqrt{d^3 / K})$ rate  as long as the dataset has sufficient coverage over the trajectory induced by the desired SNE.

\section{Related Work}\quad  
In this section, we provide an overview of some of the related work on a number of different topics in game theory and machine learning. 
 
{\noindent \bf RL  for  solving NE in Markov games}  \quad 
Our work adds to the vast body of existing literature on RL for finding Nash equilibria in  Markov games. 
In particular, 
there is a line of work that generalizes single-agent RL algorithms to Markov games under either a generative model \citep{azar2013minimax} or offline settings with well-explored datasets 
 \citep{littman2001friend,greenwald2003correlated,hu2003nash,lagoudakis2012value,hansen2013strategy,perolat2015approximate,jia2019feature,sidford2020solving, cui2020minimax,fan2020theoretical, daskalakis2021independent, zhao2021provably}. 
These works all aim to find the Nash equilibrium and their algorithms are generalizations of single-agent RL algorithms. 
In particular, \citet{littman2001friend,littman1994markov, greenwald2003correlated} and \cite{hu2003nash} generalize   Q-learning \citep{watkins1992q} to Markov games and establish asymptotic convergence guarantees.  
\citet{jia2019feature, sidford2020solving, zhang2020model} and \cite{cui2020minimax} propose variants of  Q-learning or value iteration \citep{shapley1953stochastic}  algorithms under a generative model setting. 
Also, 
\citet{perolat2015approximate,fan2020theoretical} study the sample efficiency of  fitted value iteration \citep{munos2008finite} for zero-sum Markov games in an offline setting. 
They assume the behavior policy is explorative in the sense that the concentrability coefficients \citep{munos2008finite} are uniformly bounded. 
Under similar assumptions, \citet{daskalakis2021independent} and \cite{zhao2021provably} study the sample complexity of policy gradient \citep{sutton1999policy} under the well-explored offline setting. 

In the online setting, 
there is a recent line of research that proposes provably efficient RL algorithms for zero-sum Markov games. 
See, e.g.,  
\citet{wei2017online, bai2020near,bai2020provable,liu2020sharp, tian2020provably, xie2020learning, chen2021almost} and the references therein. 
These works propose optimism-based algorithms and establish sublinear regret guarantees for finding NE. 
Among these works,  our work is particularly related to \citet{xie2020learning} and \citet{chen2021almost}, whose algorithms also incorporate linear function approximation.

Compared with these aforementioned works, we focus on solving the Stackelberg-Nash equilibrium, which involves a bilevel structure and is fundamentally different from the Nash equilibrium. Thus, our work is not directly comparable. 

{\noindent \bf Learning Stackelberg games} \quad 
As for solving Stackelberg-Nash  equilibrium, most of the existing results focus on the normal form game, which is equivalent to our Markov game with $H = 1$. 
\citet{letchford2009learning, blum2014learning} and \citet{peng2019learning} study  learning of Stackelberg equilibria with
a best response oracle. \citet{fiez2019convergence} 
study the local convergence of first-order methods for finding
Stackelberg equilibria in general-sum games with  differentiable reward functions, and \citet{ghadimi2018approximation, chen2021single} and \citet{hong2020two} analyze the global convergence of  first-order methods for achieving global optimality of bilevel optimization.
A more closely related paper is \citet{bai2021sample}, which studies the matrix Stackelberg game with bandit feedback. 
This work also studies an RL extension where the leader has a finite action set and the follower is faced with an MDP specified by the leader's action.   
In comparison, we assume the leader knows the reward functions and the main challenge lies in the unknown transitions. Thus, our setting is different from that in \citet{bai2021sample}.  
A more direclty relevant  work is   \citep{bucarey2019value}, who derive a Bellman equation and a value iteration algorithm for solving for SNE in Markov games. 
In comparison, we establish modifications of least-squares value iteration that are tailored to  online and offline settings. 

{\noindent \bf Learning general-sum Markov games} \quad \citet{liu2020sharp} present the first result of finding correlated equilibrium (CE) and coarse correlated equilibrium (CCE) in general-sum Markov games. However, their centralized algorithms suffer from the curse of many agents, that is, their sample complexity scales exponentially in the number of agents. Recently, \citet{song2021can,mao2021provably} and \citet{jin2021v}  escape the curse of many agents via the decentralized structure of the V-learning algorithm \citep{bai2020near}. 

{\noindent \bf Related single-agent RL methods}\quad Our work is also related to recent work that achieves sample efficiency in single-agent RL under an online setting. 
See, e.g.,  \citet{azar2017minimax,jin2018q,yang2019sample,zanette2019tighter,jin2020provably,zhou2020nearly,ayoub2020model,yang2020reinforcement,zanette2020learning,zanette2020frequentist,zhang2020almost,zhang2020reinforcement,agarwal2020model} and the references therein. 
In particular,  following the \emph{optimism in the face of uncertainty} principle, these works achieve near-optimal regret under either tabular or function approximation settings. 
Meanwhile, for offline RL with an arbitrary dataset, various recent works propose to utilize pessimism 
for achieving robustness.  
See, e.g., 
  \citet{yu2020mopo,kidambi2020morel,kumar2020conservative,jin2020pessimism,liu2020provably,buckman2020importance,rashidinejad2021bridging}
  and the references therein. 
 These aforementioned works all focus on the single-agent setting and we prove that optimism and pessimism also play an indispensable role in achieving sample efficiency in finding SNE.

 {\noindent \bf Notation.}
We denote by $\| \cdot \|_2$ the $\ell_2$-norm of a vector or the spectral norm of a matrix. We also let $\|\cdot\|_{\text{op}}$ denote the matrix operator norm. Furthermore, for a positive definite matrix $A$, we denote by $\|x\|_A$ the weighted norm $\sqrt{x^\top Ax}$ of a vector $x$. Also, we denote by $\Delta(\cA)$ the set of probability distributions on a set $\cA$. For some positive integer $K$, $[K]$ denotes the index set $ \{1, 2, \cdots, K\}$. 

\section{Preliminaries}

In this section, we introduce the formulation of the  general-sum simultaneous-move 
Markov games and the notion of a Stackelberg-Nash equilibrium.  We also present the linear Markovian structure that we study in this paper. 

\subsection{General-Sum Simultaneous-Move Markov Games}
In this setting, two levels of hierarchy in decision making are considered: one leader $l$ and $N$ followers $\{f_i\}_{i \in [N]}$. Specifically, we define an episodic version of general-sum simultaneous-moves Markov game by the tuple $(\cS, \cA_l, \cA_f = \{\cA_{f_i}\}_{i \in [N]}, H, r_l, r_f = \{r_{f_i}\}_{i \in [N]}, \cP)$, where $\cS$ is the state space, $\cA_l$ and $\cA_f$ are the sets of actions of the leader and the followers, respectively, $H$ is the number of steps in each episode, $r_l = \{r_{l, h}: \cS \times \cA_l \times \cA_{f} \rightarrow [-1, 1]\}_{h=1}^H$ and  $r_{f_i} = \{r_{f_i, h}: \cS \times \cA_l \times \cA_{f} \rightarrow [-1, 1]\}_{h=1}^H$ are reward functions of the leader and the followers, respectively, and $\cP = \{\cP_h: \cS \times \cA_l \times \cA_{f} \times \cS \rightarrow [0, 1]\}_{h=1}^H$ is a collection of transition kernels. Here $\cA_l \times \cA_f = \cA_l \times \cA_{f_1} \times \cdots \times \cA_{f_N}$. Throughout this paper, we also let $\star$ be some element in $\{l, f_1, \cdots, f_N\}$. Finally, for any $(h, x, a) \in [H] \times \cS \times \cA_l$ and $b = \{b_i \in \cA_{f_i}\}_{i \in [N]}$, we use the shorthand $r_{\star, h}(x, a, b) = r_{\star, h}(x, a, b_{1}, \cdots, b_{N})$ and $\cP_h(\cdot \,|\, x, a, b) = \cP_h(\cdot \,|\, x, a, b_{1}, \cdots, b_{N})$.

\vskip 4pt
\noindent{\bf Policy and Value Function.}
A stochastic policy $\pi = \{\pi_h : \cS \rightarrow \Delta(\cA_l) \}_{h = 1}^H$ of the leader is a set of probability distributions over actions given the state. Meanwhile, a stochastic joint policy of the followers is defined by $\nu = \{\nu_{f_i}\}_{i \in [N]}$, where $\nu_{f_i} = \{\nu_{f_i, h} : \cS \rightarrow \Delta(\cA_{f_i}) \}_{h = 1}^H$. We use the notation $\pi_h(a\,|\,x)$ and $\nu_{f_i, h}(b_i\,|\,x)$ to denote the probability of taking action $a \in \cA_l$ or $b_i \in \cA_{f_i}$ for state $x$ at step $h$ under policy $\pi, \nu_{f_i}$ respectively. Throughout this paper, for any $\nu = \{\nu_{f_i}\}_{i \in [N]}$ and $b = \{b_i\}_{i \in [N]}$, we use the shorthand $\nu_h(b \,|\, x) = \nu_{f_1, h}(b_1 \,|\, x) \times \cdots \times \nu_{f_N, h}(b_N \,|\, x)$.

Given policies $(\pi, \nu = \{\nu_{f_i}\}_{i \in [N]})$, the action-value (Q) and state-value (V)  functions for the leader and followers are defined by 
{\small
\begin{equation}
\begin{aligned} \label{eq:def:value:leader}
Q_{\star, h}^{\pi, \nu}(x, a, b) = \EE_{\pi, \nu, h, x, a, b} \bigg[\sum_{t = h}^H r_{\star, h}(x_t, a_t, b_t)  \biggr], \quad  V_{\star, h}^{\pi, \nu}(x) = \EE_{a \sim \pi_h(\cdot\,|\, x), b \sim \nu_{h}(\cdot\,|\,x)} Q_{\star, h}^{\pi, \nu}(x, a, b),  
\end{aligned}
\end{equation}}
where the expectation $\EE_{\pi, \nu, h, x, a, b}$ is taken over state-action pairs induced by the policies $(\pi, \nu = \{\nu_{f_i}\}_{i \in [N]})$ and the transition probability, when initializing the process with the triplet $(s, a, b = \{b_i\}_{i \in [N]})$ at step $h$. For notational simplicity, when $h, x, a, b$ are clear from the context, we omit $h, x, a, b$ from $\EE_{\pi, \nu, h, x, a, b}$.
By the definition in  \eqref{eq:def:value:leader}, we have the Bellman equation
\begin{equation}
\begin{aligned} \label{eq:bellman}
V_{\star, h}^{\pi, \nu} = \la Q_{\star, h}^{\pi, \nu}, \pi_h \times \nu_h \ra_{\cA_l \times \cA_f}, \quad Q_{\star, h}^{\pi, \nu} = r_{\star, h} + \PP_h V_{\star, h+1}^{\pi, \nu}, \quad \forall \star \in \{l, f_1, \cdots, f_N\} ,
\end{aligned} 
\end{equation}
where $\pi_h \times \nu_h$ represents  $\pi_h \times \nu_{{f_1}, h} \times \cdots \times \nu_{f_N, h}$. 
Here $\PP_h$ is the operator which is defined by 
\#
(\PP_h f)(x, a, b) = \EE[f(x') \,|\, x' \sim \cP_h(x' \,|\, x, a, b)]
\#
for any function $f : \cS \rightarrow \RR$ and $(x, a, b) \in \cS \times \cA_l \times \cA_f$.

\subsection{Stackelberg-Nash Equilibrium} \label{sec:def:stackelberg:nash}

Given a leader policy $\pi$, a Nash equilibrium \citep{nash20167} of the followers is a joint policy $\nu^* = \{\nu_{f_i}^*\}_{i \in [N]}$, such that for any $x \in \cS$ and $(i, h) \in [N] \times [H]$
\# \label{eq:def:nash}
V_{f_i, h}^{\pi, \nu^*}(x) \ge V_{f_i, h}^{\pi, \nu_{f_i}, \nu_{f_{-i}}^*}(x), \quad \forall \nu_{f_i}.
\#
Here $-i$ represents all indices in $[N]$ except $i$.
For each leader policy $\pi$, we denote the set of best-response policies of the followers by BR$(\pi)$, which is defined by
\# \label{eq:def:best:response2}
\text{BR}(\pi) = \{ \nu \,|\, \nu \text{ is the NE of the followers given the leader policy } \pi\}.
\#
Given the best-response set BR$(\pi)$, we denote by $\nu^*(\pi)$ the best-case responses, which break ties in favor of the leader. This is also known as optimistic tie-breaking \citep{breton1988sequential,bucarey2019stationary}. We will discuss   pessimistic tie-breaking  \citep{conitzer2006computing} in \S \ref{sec:pess:tie:breaking}. Specifically, we define $\nu^*(\pi)$ by 
\# \label{eq:def:against2}
\nu^*(\pi) = \{\nu \in \text{BR}(\pi) \,|\, V_{l, h}^{\pi, \nu}(x) \ge V_{l, h}^{\pi, \nu'}(x), \forall x \in \cS, h \in [H], \nu' \in \text{BR}(\pi)\}.
\#
The Stackelberg-Nash equilibrium for the leader is the ``best response to the best response.'' In other words,  we want to find a leader policy $\pi$ that maximizes the  value function under the assumption that the followers always adopt $\nu^*(\pi)$, i.e., 
\# \label{eq:def:se2}
\text{SNE}_l = \{\pi  \,|\, V_{l, h}^{\pi, \nu^*(\pi)}(x) \ge V_{l, h}^{\pi', \nu^*(\pi')}(x), \forall x \in \cS, h \in [H], \pi'  \}
\#
A Stackelberg-Nash equilibrium of the general-sum game is a policy pair $(\pi^*, \nu^* = \{\nu_{f_i}^*\}_{i \in [N]})$ such that $\nu^* \in \nu^*(\pi^*)$ and $\pi^* \in \text{SNE}_l$. 
 
Our goal is to find the Stackelberg equilibrium---the leader’s optimal strategy under the assumption that  the followers always play their best response (Nash equilibrium) to the leader. Equivalently, we need to solve the following optimization problem:
\# \label{eq:def:bilevel:opt:problem}
\max_{\pi, \nu} V_{l, 1}^{\pi, \nu}(x) \qquad \text{s.t.} \, \nu \in \text{BR}(\pi).
\#
We study this challenging bilevel optimization problem in both the online setting (Section~\ref{sec:online}) and the offline setting (Section \ref{sec:offline}) respectively.

\subsection{Linear Markov Games with Myopic Followers} \label{sec:def:myopic:follower}

\noindent{\bf Linear Markov Games.}
We study linear Markov games \citep{xie2020learning}, where the transition dynamics are linear in a feature map. 
\begin{assumption}[Linear Markov Games] \label{assumption:linear}
	Markov game $(\cS, \cA_l, \cA_f = \{\cA_{f_i}\}_{i \in [N]}, H, r_l, r_f = \{r_{f_i}\}_{i \in [N]}, \cP)$ is a linear Markov game if there exists a feature map, $\phi: \cS \times \cA_l \times \cA_f \rightarrow \mathbb{R}^d$, such that 
	\$
	 \cP_h(\cdot \,|\, x,a,b) = \la \phi(x, a, b), \mu_h(\cdot) \ra
	\$
	for any $(x, a, b) \in \cS \times \cA_l \times \cA_{f}$ and $h \in [H]$. Here $\mu_h = (\mu_h^{(1)}, \mu_h^{(2)}, \cdots, \mu_h^{(d)})$ are $d$ unknown signed measures over $\cS$. Without loss of generality, we assume that $\|\phi(\cdot, \cdot, \cdot)\|_2 \le 1$ and $\|\mu_h(\cS)\| \le \sqrt{d}$ for all $h \in [H]$. 
	\end{assumption}

This linear Markov game is an extension of the linear MDP studied in \citet{jin2020provably} for single-agent RL. Specifically, when the followers play fixed and known policies, the linear Markov game reduces to the linear MDP. We also remark that \citet{chen2021almost} recently  study another variant of linear Markov games. These  two variants are incomparable in the sense that one does not imply the other.


\vspace{4pt}
\noindent{\bf{Myopic Followers.}} 
Throughout this paper, we make the following assumption.
\begin{assumption}[Myopic followers] \label{assumption:myopic}
	We assume that the followers are myopic.
Specifically, the followers at any step of the game do not consider the future rewards, but only the instantaneous rewards. Formally, given a leader's policy $\pi$, the NE of the myopic followers is defined by the joint policy $\nu^* = \{\nu_{f_i}^*\}_{i \in [N]}$, such that for any $x \in \cS$ and $(i, h) \in [N] \times [H]$
\# \label{eq:def:nash:myopic}
r_{f_i, h}^{\pi, \nu^*}(x) \ge r_{f_i, h}^{\pi, \nu_{f_i}, \nu_{f_{-i}}^*}(x), \quad \forall \nu_{f_i}.
\#
\end{assumption}

In other words, at each state $x$, 
the followers play a normal form game where the payoff matrices are determined only by the immediate reward functions and the leader's policy. 
Then, with slight abuse of notation, the best response set of the leader's policy $\pi$ is 
\# \label{eq:def:best:response:myopic}
\text{BR}(\pi) = \{\nu \,|\, \nu \text{ is the NE of the myopic followers given the leader policy } \pi\}.
\#
The best-case response $\nu^*(\pi)$ and Stackelberg-Nash equilibria $\text{SNE}_l$ follow the definitions in \eqref{eq:def:against2} and \eqref{eq:def:se2}.

\vspace{4pt}
\noindent{\bf{Leader-Controller Linear Markov Games.}}
A special case of the Markov games with myopic followers is leader-controller Markov games \citep{filar2012competitive,bucarey2019stationary}, where the future state only depends on the current state and the leader's action. Such a setting is also well-motivated; one can consider a game where the leader is the government that dictates prices and the followers are companies. This is a leader-controller Markov game because the future state (price) is determined by the current state (price) and the leader (government). Formally, it holds that  
$\cP_h( \cdot \,|\,  x, a, b) = \cP_h(\cdot \,|\, x, a)$  
for any $(x, a, b) \in \cS \times \cA_l \times \cA_{f}$ and $h \in [H]$. Hence, with slight abuse of notation, it is natural to make the following assumption.
\begin{assumption}[Leader-Controller Linear Markov Games] \label{assumption:leader:controller}
The Markov game $(\cS, \cA_l, \cA_f = \{\cA_{f_i}\}_{i \in [N]}, H, r_l, r_f = \{r_{f_i}\}_{i \in [N]}, \cP)$ is a leader-controller linear Markov games if we assume the existence of a feature map $\phi: \cS \times \cA_l \rightarrow \RR^d$ such that
\# \label{eq:def:leader:controller}
\cP_h( \cdot \,|\,  x, a, b) = \la \phi(s, a), \mu_h(\cdot) \ra,
\#
for any $(s, a, b) \in \cA \times \cA_l \times \cA_f$, where $\|\phi(\cdot, \cdot)\|_2 \le 1$ and $\|\mu_h(\cS)\| \le \sqrt{d}$ for all $h \in [H]$. 
\end{assumption}

Notably, Markov games with myopic followers subsume leader-controller Markov games as a special case. Specifically, since the followers' policies cannot affect the following state, then the NE defined in \eqref{eq:def:nash} is the same as \eqref{eq:def:nash:myopic}, which further implies that the best-response set defined in \eqref{eq:def:best:response2} is the same as \eqref{eq:def:best:response:myopic}.



\section{Main Results for the Online Setting} \label{sec:online}

In this section, we study the online setting, where a central controller controls one leader $l$ and $N$ followers $\{f_i\}_{i \in [N]}$. Our goal is to learn a Stackelberg-Nash equilibrium. In what follows, we formally describe the setup and learning objectives, and then present our algorithm and provide theoretical guarantees.

\subsection{Setup and Learning Objective}
We consider the setting where the reward functions $r_l$ and $r_{f} = \{r_{f_i}\}_{i \in [N]}$ are revealed to the learner before the game. This is reasonable since in practice the reward functions are usually artificially designed. Moreover, we focus on the episodic setting. Specifically, a Markov game is played for $K$ episodes, each of which consists of $H$ timesteps. At the beginning of the $k$-th episode, the leader and followers determine their policies $(\pi^k, \nu^k = \{\nu_{f_i}^k\}_{i \in [N]})$, and a fixed initial state $x_1^k = x_1$ is chosen. Here we assume a fixed initial state for ease of presentation---our subsequent results can be readily generalized to the setting where $x_1^k$ is picked from a fixed distribution. The game proceeds as follows. At each step $h \in [H]$, the leader and the followers observe state $x_h^k \in \cS$ and pick their own actions $a_h^k \sim \pi_h^k(\cdot \,|\, x_h^k)$ and $b_h^k = \{b_{i, h}^k \sim \nu_{f_i, h}^k(\cdot \,|\, x_h^k)\}_{i \in [N]}$. Subsequently, the environment transitions to the next state $x_{h+1}^k \sim \cP_h(\cdot \,|\, x_h^k, a_h^k, b_{h}^k)$. Each episode terminates after $H$ timesteps. 

\vskip 4pt
\noindent{\bf Learning Objective.} 
Let $(\pi^k, \nu^k = \{\nu_{f_i}^k\}_{i \in [N]})$ denote the policies executed by the algorithm in the $k$-th episode. By the definition of the bilevel optimization problem in \eqref{eq:def:bilevel:opt:problem}, we expect that $\nu^k \in \text{BR}(\pi^k)$ and that $V_{l,1}^{\pi^*, \nu^*}(x_1^k) - V_{l, 1}^{\pi^k, \nu^k}(x_1^k)$ is small for any $k \in [K]$. Hence, we evaluate the suboptimality of our algorithm by the regret, which is defined as    
\# \label{eq:def:regret2}
\text{Regret}(K) = \sum_{k = 1}^K V_{l, 1}^{\pi^*, \nu^*}(x_1^k) - V_{l, 1}^{\pi^k, \nu^k}(x_1^k).
\#
The goal is to design algorithms with regret that is sublinear in $K$, and polynomial in $d$ and $H$. Here $K$ is the number of episodes, $d$ is the dimension of the feature map $\phi$, and $H$ is the episode horizon.

\subsection{Algorithm}
We now present our algorithm, Optimistic Value Iteration to Find Stackelberg-Nash Equilibria (OVI-SNE), which is given in Algorithm \ref{alg2}.

At a high level, in each episode our algorithm first constructs the policies for all players through backward induction with respect to the timestep $h$ (Line \ref{line:for:begin}-\ref{line:for:end}), and then executes the policies to play the game (Line \ref{line:execute:begin}-\ref{line:execute:end}).

In detail, at the $h$-th step of the $k$-th episode, OVI-SNE estimates the leader's Q-function based on the $(k - 1)$ historical trajectories. Inspired by previous optimistic least square value iteration (LSVI) algorithms \citep{jin2020provably}, for any $h \in [H]$, we estimate the linear coefficients by solving the following ridge regression problem:
\begin{equation}
\begin{aligned} \label{eq:least:square}
&w_{h}^k \leftarrow \argmin_{w \in \RR^d} \sum_{\tau = 1}^{k - 1} [V_{h+1}^k(x_{h+1}^\tau) - \phi(x_h^\tau, a_h^\tau, b_h^\tau)^\top w]^2 + \|w\|_2^2, \\
& \text{where }  V_{h+1}^k(\cdot) = \la Q_{h+1}^k(\cdot, \cdot, \cdot), \pi_{h+1}^k(\cdot \,|\, \cdot) \times \nu_{h+1}^k(\cdot \,|\, \cdot) \ra_{\cA_l \times \cA_{f}}.
\end{aligned}
\end{equation}
This yields the following solution:
\begin{equation}
\begin{aligned} \label{eq:def:w}
&w_{h}^k = (\Lambda_{h}^k)^{-1}\Bigl(\sum_{\tau = 1}^{k - 1}\phi(x_h^\tau, a_h^\tau, b_h^\tau) \cdot V_{h+1}^k(x_{h+1}^\tau)\Bigr), \\
& \text{where } \Lambda_{h}^k = \sum_{\tau = 1}^{k - 1} \phi(x_h^\tau, a_h^\tau, b_h^\tau) \phi(x_h^\tau, a_h^\tau, b_h^\tau)^\top + I.
\end{aligned}
\end{equation}
To encourage exploration, we additionally add a bonus function in estimating the leader's Q-function:
\begin{equation}
\begin{aligned} \label{eq:evaluation}
&Q_{h}^k(\cdot, \cdot, \cdot) \leftarrow r_{l, h}(\cdot, \cdot, \cdot) + \Pi_{H - h}\{\phi(\cdot, \cdot, \cdot)^\top w_{h}^k + \Gamma_{h}^k(\cdot, \cdot, \cdot)\}, \\
& \text{where } \Gamma_{h}^k(\cdot, \cdot, \cdot) = \beta \cdot \sqrt{\phi(\cdot, \cdot, \cdot)^\top(\Lambda_{h}^k)^{-1}\phi(\cdot, \cdot, \cdot)},
\end{aligned}
\end{equation}
where $\Gamma_{h}^k: \cS \times \cA_l \rightarrow \RR$ is a bonus and $\beta > 0$ is a parameter which will be specified later. This form of bonus function is common in the literature of linear bandits \citep{lattimore2020bandit} and linear MDPs \citep{jin2020provably}. 


Next, we construct policies for the leader and followers by the subroutine $\epsilon$-SNE (Algorithm \ref{alg22}). Specifically, let $\cQ_h^k$ be the class of functions $Q: \cS \times \cA_l \times \cA_f \rightarrow \RR$ that takes the form
\# \label{eq:def:epsilon:net}
Q(\cdot, \cdot, \cdot) = r_{l, h}(\cdot, \cdot, \cdot) + \Pi_{H - h}\big\{\phi(\cdot, \cdot, \cdot)^\top w + \beta \cdot \big(\phi(\cdot, \cdot, \cdot)^\top \Lambda^{-1} \phi(\cdot, \cdot, \cdot)\big)^{1/2}\big\}, 
\#
where the parameters $(w, \Lambda) \in \RR^d \times \RR^{d \times d}$ satisfy $\|w\| \le H\sqrt{dk}$ and $\lambda_{\min}(\Lambda) \geq 1$. Let $\cQ_{h, \epsilon}^k$ be a fixed $\epsilon$-covering of $\cQ_h^k$ with respect to the $\ell_{\infty}$ norm. By Lemma \ref{lemma:bounded:weight1}, we have $Q_h^k \in \cQ_h^k$, which allows us to pick a $\tilde{Q} \in \cQ_{h, \epsilon}^k$ such that $\|\tilde{Q} - Q_h^k\|_{\infty} \le \epsilon$ and calculate policies by 
\# \label{eq:def:matrix:game}
(\pi_h^k(\cdot\,|\,x), \{\nu_{f_i,h}^k(\cdot\,|\, x)\}_{i \in [N]}) \leftarrow {\text{SNE}}(\tilde{Q}(x, \cdot, \cdot), \{r_{f_i, h}(x, \cdot, \cdot)\}_{i \in [N]}), \, \forall x.
\#
When there is only one follower, \eqref{eq:def:matrix:game} requires finding a Stackelberg equilibrium in the matrix game. Such a problem can be transformed to a linear programming (LP) problem \citep{conitzer2006computing,von2010leadership}, and thus can be solved efficiently. 
For the multi-follower case (i.e., $N \ge 2$), however, solving such a matrix game in general is hard \citep{conitzer2006computing,basilico2017methods,basilico2017bilevel,coniglio2020computing}. Given this computational hardness, we focus on the sample complexity and explicitly assume access to the following computational oracle:
\begin{assumption} \label{assumption:oracle}
	We assume access to an oracle that implements Line \ref{line:matrix:game} of Algorithm \ref{alg22}. 
\end{assumption} 



Finally, the leader and the followers play the game according to the obtained policies.

\begin{algorithm}[H] 
	\caption{Optimistic Value Iteration to Find Stackelberg-Nash Equilibria}
	\begin{algorithmic}[1] \label{alg2}
    	\STATE Initialize $V_{l, H+1}(\cdot) = V_{f, H+1}(\cdot) = 0$.
        \FOR{$k = 1, 2, \cdots, K$}
        \STATE Receive initial state $x_1^k$.
    	\FOR{step $h = H, H-1, \cdots, 1$} \label{line:for:begin}
    	\STATE $\Lambda_{h}^k \leftarrow \sum_{\tau = 1}^{k - 1} \phi(x_h^\tau, a_h^\tau, b_h^\tau) \phi(x_h^\tau, a_h^\tau, b_h^\tau)^\top + I$. \label{line:def:lambda}
		\STATE $w_{h}^k \leftarrow (\Lambda_{h}^k)^{-1}\sum_{\tau = 1}^{k - 1}\phi(x_h^\tau, a_h^\tau, b_h^\tau) \cdot V_{h+1}^k(x_{h+1}^\tau)$. \label{line:def:w}
		\STATE $\Gamma_{h}^k(\cdot, \cdot, \cdot) \leftarrow \beta \cdot (\phi(\cdot, \cdot, \cdot)^\top(\Lambda_h^k)^{-1}\phi(\cdot, \cdot, \cdot))^{1/2}$. \label{line:def:bonus}
		\STATE $Q_{h}^k(\cdot, \cdot, \cdot) \leftarrow r_{l, h}(\cdot, \cdot, \cdot) + \Pi_{H - h}\{\phi(\cdot, \cdot, \cdot)^\top w_{h}^k + \Gamma_{h}^k(\cdot, \cdot, \cdot)\}$. \label{line:def:q}
		\STATE  $(\pi_h^k(\cdot\,|\,x), \{\nu_{f_i,h}^k(\cdot\,|\, x)\}_{i \in [N]}) \leftarrow \epsilon$-${\text{SNE}}(Q_{h}^k(x, \cdot, \cdot), \{r_{f_i, h}(x, \cdot, \cdot)\}_{i \in [N]})$, $\, \forall x$. (Alg. \ref{alg22}) \label{line:policy}
		\STATE $V_{h}^k(x) \leftarrow \EE_{a \sim \pi_h^k(\cdot\,|\,x), b_1 \sim \nu_{f_1, h}^k(\cdot\,|\,x), \cdots, b_N \sim \nu_{f_N, h}^k(\cdot\,|\,x)}Q_{h}^k(x, a, b_1, \cdots, b_N)$, $\, \forall x$. \label{line:v:l2}
    	\ENDFOR \label{line:for:end}
        \FOR{$h = 1, 2, \cdot, H$} \label{line:execute:begin}
        \STATE Sample $a_h^k \sim \pi_h^k(\cdot\,|\,x_h^k), b_{1, h}^k \sim \nu_{f_1, h}^k(\cdot\,|\,x_h^k), \cdots, b_{N, h}^k \sim \nu_{f_N, h}^k(\cdot\,|\,x_h^k)$.
        \STATE Leader takes action $a_h^k$; Followers take actions $b_h^k = \{b_{i, h}^k\}_{i \in [N]}$.    
        \STATE Observe next state $x_{h+1}^k$.
        \ENDFOR \label{line:execute:end}
        \ENDFOR
	\end{algorithmic}
\end{algorithm}

\begin{algorithm}[H] 
	\caption{$\epsilon$-SNE}
	\begin{algorithmic}[1] \label{alg22}
	\STATE {\bf Input:} $Q_h^k, x$, and parameter $\epsilon$.
	\STATE Select $\tilde{Q}$ from $\cQ_{h, \epsilon}^k$ satisfying $\|\tilde{Q} - Q_h^k\|_{\infty} \le \epsilon$. 
	\STATE For the input state $x$, let $(\pi_h^k(\cdot\,|\,x), \{\nu_{f_i,h}^k(\cdot\,|\, x)\}_{i \in [N]})$ be the Stackelberg-Nash equilibrium for the matrix game with payoff matrices $(\tilde{Q}(x, \cdot, \cdot), \{r_{f_i, h}(x, \cdot, \cdot)\}_{i \in [N]})$. \label{line:matrix:game}
	\STATE {\bf Output:} $(\pi_h^k(\cdot\,|\,x), \{\nu_{f_i,h}^k(\cdot\,|\, x)\}_{i \in [N]})$. 
	\end{algorithmic} 
\end{algorithm} 

Finally we explain the motivation for using the subroutine $\epsilon$-SNE to construct policies instead of solving the matrix games with payoff matrices $(Q_h^k(x, \cdot, \cdot), \{r_{f_i, h}(x, \cdot, \cdot)\}_{i \in [N]})$ directly. By the definition of $Q_h^k$ in \eqref{eq:evaluation}, we know that $Q_h^k$ relies on the previous data via the estimated value function $V_{h+1}^k$ and feature maps $\{\phi(x_h^\tau, a_h^\tau, b_h^\tau)\}_{\tau = 1}^{k - 1}$. Similar to the analysis for linear MDPs \citep{jin2020provably}, we need to use a covering argument to establish uniform concentration bounds for all value $V_{h+1}^k$. \cite{jin2020provably} directly construct an $\epsilon$-net for the value functions and establishes a polynomial log-covering number for this $\epsilon$-net. This analysis, however, relies on an assumption that the policies executed by the players are greedy (deterministic), which is not valid for our setting. To overcome this technical issue, we construct an $\epsilon$-net for Q-functions and solve an approximate matrix game. Fortunately, by choosing a small enough $\epsilon$, we can handle the errors caused by this approximation. See \S \ref{appendix:pf:thm:multi} for more details. Moreover, as shown in \citet{xie2020learning}, this subroutine can be implemented efficiently without explicitly computing the exponentially large $\epsilon$-net.

\subsection{Theoretical Results}
Our main theoretical result is the following bound on the regret incurred by Algorithm \ref{alg2}. Recall that the regret is defined in \eqref{eq:def:regret2} and $T = KH$ is the total number of timesteps.

\begin{theorem}  \label{thm:multi:follower}
    Under Assumptions \ref{assumption:linear}, \ref{assumption:myopic}, and \ref{assumption:oracle}, there exists an absolute constant $C > 0$ such that, for any fixed $p \in (0, 1)$, by setting $\beta = C \cdot dH \sqrt{\iota}$ with $\iota = \log(2dT/p)$ in Line \ref{line:def:bonus} of Algorithm~\ref{alg2} and $\epsilon = \frac{1}{KH}$ in Algorithm \ref{alg22},  we have $\nu^k \in \text{BR}(\pi^k)$ for any $k \in [K]$. Meanwhile, with probability at least $1 - p$, the regret incurred by OVI-SNE satisfies  
    \$
    \text{Regret}(K) \le \cO(\sqrt{d^3H^3T\iota^2}).
    \$
\end{theorem}

\begin{proof}
	See \S \ref{appendix:pf:thm:multi} for a detailed proof.
\end{proof}

\noindent{\bf Learning Stackelberg Equilibria.}
When there is only one follower, the Stackelberg-Nash equilibrium reduces to the Stackelberg equilibrium \citep{simaan1973stackelberg,conitzer2006computing,bai2021sample}. Thus, we partly answer the open problem in \cite{bai2021sample} on how to learn Stackelberg equilibria in general-sum Markov games (with myopic followers).

\vskip 4pt
\noindent{\bf Optimality of the Bound.}
	Assuming that the action of the follower won’t affect the transition kernel and reward function, the linear Markov games reduces to the linear MDP \citep{jin2020provably}. Meanwhile, the lower bound established in \citet{azar2017minimax} and \citet{jin2018q} for tabular MDPs and the lower bound established in \citet{lattimore2020bandit} for linear bandits directly imply a lower bound $\Omega(dH\sqrt{T})$ for the linear MDPs, which further yields a lower bound $\Omega(dH\sqrt{T})$ for our setting. Ignoring the logarithmic factors, there is only a gap of $\sqrt{dH}$ between this lower bound and our upper bound. We also point out that, by using the ``Bernstein-type'' bonus \citep{azar2017minimax,jin2018q,zhou2020nearly}, we can improve our upper bound by a factor of $\sqrt{H}$. Here we do not apply this technique for the clarity of the analysis.

\vskip 4pt
\noindent{\bf Unknown Reward Setting.}
To relax the assumption that the reward is known, we consider the case where the reward functions are unknown. 
At a high level, we first conduct a reward-free exploration algorithm (Algorithm \ref{alg5} in \S \ref{appendix:unknown:reward}), a variant of Reward-Free RL-Explore algorithm in \citet{jin2020reward}, to obtain estimated reward functions $\{\hat{r}_l, \hat{r}_{f_1}, \cdots \hat{r}_{f_N}\}$. As asserted before, we can use Algorithm \ref{alg2} to find the SNE with respect to the \emph{known} estimated reward functions $\{\hat{r}_l, \hat{r}_{f_1}, \cdots \hat{r}_{f_N}\}$. Hence, we can obtain the approximate SNE if the value functions of estimated value functions are good approximation of the true value functions. See \S \ref{appendix:unknown:reward} for the detailed algorithm and theoretical guarantees.

\section{Main Results for the Offline Setting} \label{sec:offline}
In this section, we study the offline setting, where the central controller aims to find a Stackelberg-Nash equilibrium by analyzing an offline dataset. Below we describe the setup and learning objective, followed by our algorithm and theoretical results.

\subsection{Setup and Learning Objective}
We now assume that the learner has access to the reward functions $(r_l, r_f = \{r_{f_i}\}_{i = 1}^N)$ and a dataset $\cD = \{(x_h^\tau, a_h^\tau, b_h^\tau = \{b_{i, h}^\tau\}_{i = 1}^N)\}_{\tau, h = 1}^{K, H}$, which is collected a priori by some experimenter. We make the following minimal assumption for the offline dataset. 

\begin{assumption}[Compliance of Dataset] \label{assumption:compliance:data}
	We assume that the dataset $\cD$ is compliant with the underlying Markov game $(\cS, \cA_l, \cA_f, H, r_l, r_f, \cP)$, that is, for any $x' \in \cS$ at step $h \in [H]$ of each trajectory $\tau \in [K]$,
	\$
	P_{\cD}(x_{h + 1}^\tau = x' \,|\, \{x_h^j, a_h^j, b_h^j, x_{h + 1}^j\}_{j = 1}^{\tau - 1} \cup \{x_h^\tau, a_h^\tau, b_h^\tau\})  = P(x_{h + 1} = x' \,|\, x_h = x_h^\tau, a_h = a_h^\tau, b_h = b_h^\tau).
	\$
	Here the probability on the left-hand side is with respect to the joint distribution over the dataset $\cD$ and the probability on the right-hand side is with respect to the underlying Markov game.
\end{assumption}

Assumption \ref{assumption:compliance:data} is adopted from \citet{jin2020pessimism}, and it indicates the Markov property of the dataset $\cD$ and that $x_{h +1}^\tau$ is generated by the underlying Markov game conditioned on $(x_h^\tau, a_h^\tau, b_h^\tau)$. As a special case, Assumption \ref{assumption:compliance:data} holds when the experimenter follows fixed behavior policies. More generally, Assumption \ref{assumption:compliance:data} allows the experimenter to choose actions $a_h^\tau$ and $b_h^\tau$ arbitrarily, even in an adaptive or adversarial manner. In particular, we can assume that $a_h^\tau$ and $b_h^\tau$ are interdependent across each trajectory $\tau \in [K]$. For instance, the experimenter can sequentially improve the behavior policy using any online algorithm for Markov games. 

\noindent{\bf Learning Objective.}
Similar to the online setting, we define the following performance metric
	\#  \label{eq:def:optimality:gap}
	\text{SubOpt}(\pi, \nu, x) = V_{l, 1}^{\pi^*, \nu^*}(x) - V_{l, 1}^{\pi, \nu}(x), 
	\#
which evaluates the suboptimality of policies $(\pi, \nu = \{\nu_{f_i}\}_{i = 1}^N)$ given the initial state $x \in \cS$.

\subsection{Algorithm}
While the challenge of managing the exploration-exploration tradeoff disappears in the offline setting, another statistical challenge arises: we only have access to a limited data sample. To tackle this challenge, we need add a penalty function to achieve statistical efficiency and robustness. Such a penalty can be viewed as a form of ``pessimism'' \citep{yu2020mopo,jin2020pessimism,liu2020provably,buckman2020importance,kidambi2020morel,kumar2020conservative,rashidinejad2021bridging}. Here we simply flip the sign of bonus functions defined in \eqref{eq:evaluation} to serve as pessimistic penalty functions. See Algorithm \ref{alg3} for details. 


\begin{algorithm}[H]
	\caption{Pessimistic Value Iteration to Find Stackelberg-Nash Equilibria}
	\begin{algorithmic}[1] \label{alg3}
		\STATE {\bf Input:} $\cD = \{x_h^\tau, a_h^\tau, b_h^\tau = \{b_{i, h}^\tau\}_{i \in [N]}\}_{\tau, h = 1}^{K, H}$ and reward functions $\{r_l, r_f = \{r_{f_i}\}_{i \in [N]}\}$.
    	\STATE Initialize $\hat{V}_{H+1}(\cdot) = 0$.
    	\FOR{step $h = H, H-1, \cdots, 1$}
    	\STATE $\Lambda_{h} \leftarrow \sum_{\tau = 1}^{K} \phi(x_h^\tau, a_h^\tau, b_h^\tau) \phi(x_h^\tau, a_h^\tau, b_h^\tau)^\top + I$. \label{line:def:lambda2} 
		\STATE $w_{h} \leftarrow (\Lambda_{h})^{-1}\sum_{\tau = 1}^{K}\phi(x_h^\tau, a_h^\tau, b_h^\tau) \cdot \hat{V}_{h+1}(x_{h+1}^\tau)$. \label{line:def:w2}
		\STATE $\Gamma_{h}(\cdot, \cdot, \cdot) \leftarrow \beta' \cdot (\phi(\cdot, \cdot, \cdot)^\top(\Lambda_h)^{-1}\phi(\cdot, \cdot, \cdot))^{1/2}$. \label{line:def:bonus2}
		\STATE $\hat{Q}_{h}(\cdot, \cdot, \cdot) \leftarrow r_{l, h}(\cdot, \cdot, \cdot) + \Pi_{H - h}\{\phi(\cdot, \cdot, \cdot)^\top w_{h} - \Gamma_{h}(\cdot, \cdot, \cdot)\}$. \label{line:def:q2}
		\STATE  $(\hat{\pi}_h(\cdot\,|\,x), \{\hat{\nu}_{f_i,h}(\cdot\,|\, x)\}_{i \in [N]}) \leftarrow \epsilon$-SNE$(\hat{Q}_{h}(x, \cdot, \cdot), \{r_{f_i, h}(x, \cdot, \cdot)\}_{i \in [N]})$, $\, \forall x$. (Alg. \ref{alg22}) \label{line:sne2}
		\STATE $\hat{V}_{h}(x) \leftarrow \EE_{a \sim \hat{\pi}_h(\cdot\,|\,x), b_1 \sim \hat{\nu}_{f_1, h}(\cdot\,|\,x), \cdots, b_N \sim \hat{\nu}_{f_N, h}(\cdot\,|\,x)}\hat{Q}_{h}(x, a, b_1, \cdots, b_N)$, $\, \forall x$. \label{line:v:l3}
    	\ENDFOR
		\STATE {\bf Output:} $(\hat{\pi} = \{\hat{\pi}_h\}_{h = 1}^H, \hat{\nu} = \{\hat{\nu}_{f_i} = \{\nu_{f_i, h}\}_{h = 1}^H\}_{i = 1}^N)$.
	\end{algorithmic}
\end{algorithm}

\subsection{Theoretical Results}

Suppose that $(\hat{\pi} = \{\hat{\pi}_h\}_{h = 1}^H, \hat{\nu} = \{\hat{\nu}_{f_i} = \{\nu_{f_i, h}\}_{h = 1}^H\}_{i = 1}^N)$ are the output policies of Algorithm~\ref{alg3}. We evaluate the performance of $(\hat{\pi}, \hat{\nu})$ by establishing an upper bound for the optimality gap defined in \eqref{eq:def:optimality:gap}. 

\begin{theorem}  \label{thm:offline} 
    Under Assumptions \ref{assumption:linear}, \ref{assumption:myopic}, \ref{assumption:oracle}, and \ref{assumption:compliance:data}, there exists an absolute constant $C > 0$ such that, for any fixed $p \in (0, 1)$, by setting $\beta' = C \cdot dH\sqrt{\log(2dHK/p)}$ in Line \ref{line:def:bonus2} of Algorithm~\ref{alg3} and $\epsilon = \frac{d}{KH}$ in Algorithm \ref{alg22}, then it holds that $\hat{\nu} \in \text{BR}(\hat{\pi})$. Meanwhile, with probability at least $1 - p$, we have 
    \# \label{eq:offline:thm}
    \text{SubOpt}(\hat{\pi}, \hat{\nu}, x) \le 3\beta' \sum_{h = 1}^H \EE_{\pi^*, \nu^*, x} \bigl[  \bigl(\phi(s_h, a_h, b_h)^\top(\Lambda_h)^{-1}\phi(s_h, a_h, b_h)\bigr)^{1/2}  \bigr],
    \#
	where $\EE_{\pi^*, \nu^*, x}$ is taken with respect to the trajectory incurred by $(\pi^*, \nu^*)$ in the underlying Markov game when initializing the progress at $x$. Here $\Lambda_h$ is defined in Line \ref{line:def:lambda2} of Algorithm~\ref{alg3}.
\end{theorem}

\begin{proof}
	See \S \ref{appendix:pf:offline} for a detailed proof. 
\end{proof}

We provide some discussion of the implications of Theorem \ref{thm:offline}:
\vskip 4pt
\noindent{\bf Minimal Assumption Requirement:}
Theorem \ref{thm:offline} only relies on the compliance of the dataset with linear Markov games. Compared with existing literature on offline RL \citep{bertsekas1996neuro,antos2007fitted,antos2008learning,munos2008finite,farahmand2010error,farahmand2016regularized,scherrer2015approximate,liu2018breaking,chen2019information,fan2020theoretical,xie2020q}, we impose no restrictions on the coverage of the dataset. Meanwhile, we need no assumption on the affinity between $(\hat{\pi}, \hat{\nu})$ and the behavior policies that induce the dataset, which is often employed as a regularizer \citep{fujimoto2019off,laroche2019safe,jaques2019way,wu2019behavior,kumar2019stabilizing,wang2020critic,siegel2020keep,nair2020accelerating,liu2020provably}.

\vskip 4pt
\noindent{\bf Dataset with Sufficient Coverage:}
In what follows, we specialize Theorem \ref{thm:offline} to the setting where we assume the dataset with good ``coverage.''
Note that $\Lambda_h$ is determined by the offline dataset $\cD$ and acts as a fixed matrix in the expectation, that is, the expectation in \eqref{eq:offline:thm} is only taken with the trajectory induced by $(\pi^*, \nu^*)$. As established in the following theorem, when the trajectory induced by $(\pi^*, \nu^*)$ is ``covered'' by the dataset $\cD$ sufficiently well, we can establish that the suboptimality incurred by Algorithm~\ref{alg3} diminishes at rate of $\tilde{\cO}(1/\sqrt{K})$.

\begin{corollary} \label{cor:suff:cover}
	Suppose it holds with probability at least $1 - p/2$ that  
	\$ 
    \Lambda_h \succeq I + c \cdot K \cdot \EE_{\pi^*, \nu^*, x}[\phi(s_h, a_h, b_h)\phi(s_h, a_h, b_h)^\top],
	\$
	for all $(x, h) \in \cS \times [H]$. Here $c > 0$ is an absolute constant and $\EE_{\pi^*, \nu^*, x}$ is taken with respect to the trajectory incurred by $(\pi^*, \nu^*)$ in the underlying Markov game when initializing the progress at $x$. Under Assumptions \ref{assumption:linear}, \ref{assumption:myopic}, \ref{assumption:oracle} and \ref{assumption:compliance:data}, there exists an absolute constant $C > 0$ such that, for any fixed $p \in (0, 1)$, by setting $\beta' = C \cdot dH\sqrt{\log(4dHK/p)}$ in Line \ref{line:def:bonus2} of Algorithm~\ref{alg3} and $\epsilon = \frac{d}{KH}$ in Algorithm \ref{alg22}, then it holds with probability at least $1 - p$ that 
	\$
	\text{SubOpt}(\hat{\pi}, \hat{\nu}, x) \le \bar{C} \cdot d^{3/2}H^2\sqrt{\log(4dHK/p)/K},
	\$
	for all $x \in \cS$. Here $\bar{C}$ is another absolute constant that only depends on $c$ and $C$.
\end{corollary}

\begin{proof}
	See \S \ref{appendix:pf:cor:suff:cover} for a detailed proof.
\end{proof}

Note that, unlike the previous literature which relies on a ``uniform coverage'' assumption \citep{antos2007fitted,munos2008finite,farahmand2010error,farahmand2016regularized,scherrer2015approximate,liu2018breaking,chen2019information,fan2020theoretical,xie2020q}, Corollary \ref{cor:suff:cover} only assumes that the dataset has good coverage of the trajectory incurred by the policies $(\pi^*, \nu^*)$.

\vskip 4pt
\noindent{\bf Optimality of the Bound:}
Fix $x \in \cS$. Assuming $r_l = r_{f_i}$ for any $i \in [N]$, we know $(\pi^*, \nu^*) = \argmax_{\pi, \nu}V_{l, 1}^{\pi, \nu}(x)$. Then the information-theoretic lower bound for offline single-agent RL (e.g., Theorem 4.7 in \citet{jin2020pessimism}) can imply the information-theoretic lower bound $\Omega(\sum_{h=1}^H\EE_{\pi^*, \nu^*, x} [ (\phi(s_h, a_h, b_h)^\top(\Lambda_h)^{-1}\phi(s_h, a_h, b_h))^{1/2}])$ for our setting. 
In particular, our upper bound established in Theorem \ref{thm:offline} matches this lower bound up to $\beta'$ and absolute constants and thus implies that our algorithm is nearly minimax optimal.

\section{Conclusion}
In this paper, we investigate the question of can we  efficiently find SNE in general-sum Markov games with myopic followers and linear function approximation. To the best of our knowledge, our work provides the first sample-efficient reinforcement learning algorithms for solving SNE in both online and offline settings. We believe our work opens up many interesting directions for future work. For example, we can ask the following questions: Can we find SNE in general-sum Markov games without the myopic followers assumption? Can we design more computationally efficient algorithms for solving SNE in general-sum Markov games? Can we find SNE in general-sum Markov games  with general function approximation?

\newpage
\bibliographystyle{ims}
\bibliography{graphbib}
\newpage

\begin{appendix}

\section{Proof of Theorem \ref{thm:multi:follower}} \label{appendix:pf:thm:multi}

\begin{proof}[Proof of Theorem \ref{thm:multi:follower}]
    By the myopic followers assumption, we have the following lemma.
    \begin{lemma} \label{lemma:best:response}
        For any $k \in [K]$, we have $\nu^k \in \text{BR}(\pi^k)$. Here $\text{BR}(\cdot)$ is defined in \eqref{eq:def:best:response:myopic}. 
    \end{lemma}

    \begin{proof}
        Combining the definition of $(\pi^k, \nu^k)$ in Line \ref{line:policy} of Algorithm \ref{alg2} and the definition of the best response set in the Markov games with myopic followers in \eqref{eq:def:best:response:myopic}, we conclude the proof.
    \end{proof}

    We now establish an upper bound for the regret defined in \eqref{eq:def:regret2}. Recall the regret takes the following form
    \# \label{eq:320001} 
    \text{Regret}(K) = \sum_{k = 1}^K V_{l, 1}^{\pi^*, \nu^*}(x_1^k) - V_{l, 1}^{\pi^k, \nu^k}(x_1^k) .
    \#
    To facilitate our analysis, for any $(k, h) \in [K] \times [H]$ we define the model prediction error by  
    \# \label{eq:def:td:error2}
    &\delta_{h}^k = r_{l, h} + \PP_h V_{h+1}^k - Q_{h}^k.
    \#
    Moreover, for any $(k, h) \in [K] \times [H]$, we define $\zeta_{k, h}^1$ and $\zeta_{k, h}^2$ as 
    \begin{equation}
    \begin{aligned} \label{eq:def:martingale2}
    \zeta_{k, h}^1 &= [V_{h}^k(x_h^k) - V_{l, h}^{\pi^k, \nu^k}(x_h^k)] - [Q_{h}^k(x_h^k, a_h^k, b_{h}^k) - Q_{l, h}^{\pi^k, \nu^k}(x_h^k, a_h^k, b_{h}^k)], \\
    \zeta_{k, h}^2 &= [(\PP_h V_{h+1}^k)(x_h^k, a_h^k, b_{h}^k) - (\PP_h V_{l, h+1}^{\pi^k, \nu^k})(x_h^k, a_h^k, b_{h}^k)] - [V_{h+1}^k(x_{h+1}^k) - V_{l, h+1}^{\pi^k, \nu^k}(x_{h+1}^k)].
    \end{aligned}
    \end{equation}
    Recall that $(\pi^k, \nu^k = \{\nu_{f_i}^k\}_{i \in [N]})$ are the policies executed by the leader and the followers in the $k$-th episode, which generate a trajectory $\{x_h^k, a_h^k, b_h^k = \{b_{i, h}^k\}_{i \in [N]}\}_{h \in [H]}$. Thus, we know that $\zeta_{k, h}^1$ and $\zeta_{k, h}^2$ characterize the randomness of choosing actions $a_h^k \sim \pi_h^k(\cdot \,|\, x_h^k)$ and $b_h^k \sim \nu_h^k (\cdot \,|\, x_h^k)$ and the randomness of drawing the next state $x_{h + 1}^k \sim \cP_h(\cdot \,|\, x_h^k, a_h^k, b_h^k)$, respectively.
    
    To establish an upper bound for \eqref{eq:320001}, we introduce the following lemma, which decomposes this term into three parts using the notation defined above.

    \begin{lemma}[Regret Decomposition]  \label{lemma:decomposition2}
        We can decompose \eqref{eq:320001} as follows.
        \$
        \text{Regret}(K) & =   \underbrace{\sum_{k = 1}^K\sum_{h=1}^H\EE_{\pi^*, \nu^*}[\la Q_{h}^k(x_h^k, \cdot, \cdot), \pi_h^*(\cdot \,|\, x_h^k) \times \nu_h^*(\cdot \,|\, x_h^k) -  \pi_h^k(\cdot \,|\, x_h^k) \times \nu_h^k(\cdot \,|\, x_h^k) \ra]}_{{(l.1):} \text{ Computational Error}} \\
        & \quad + \underbrace{\sum_{k = 1}^K\sum_{h=1}^H \bigl(\EE_{\pi^*, \nu^*}[\delta_{h}^k(x_h, a_h, b_{h})] - \delta_{h}^k(x_h^k, a_h^k, b_{h}^k) \bigr)}_{(l.2): \text{ Statistical Error}} + \underbrace{\sum_{k = 1}^K\sum_{h = 1}^H (\zeta_{k, h}^1 + \zeta_{k, h}^2)}_{(l.3): \text{ Randomness}}, 
        \$
        where $\la Q_{h}^k(x_h^k, \cdot, \cdot), \pi_h^*(\cdot \,|\, x_h^k) \times \nu_h^*(\cdot \,|\, x_h^k) -  \pi_h^k(\cdot \,|\, x_h^k) \times \nu_h^k(\cdot \,|\, x_h^k) \ra = \la Q_{h}^k(x_h^k, \cdot, \cdot, \cdots, \cdot), \pi_h^*(\cdot \,|\, x_h^k) \times \nu_{f_1, h}^*(\cdot \,|\, x_h^k) \times \cdots \nu_{f_N, h}^*(\cdot \,|\, x_h^k) -  \pi_h^k(\cdot \,|\, x_h^k) \times \nu_{f_1, h}^k(\cdot \,|\, x_h^k) \times \cdots \nu_{f_N, h}^k(\cdot \,|\, x_h^k) \ra_{\cA_l \times \cA_{f}}$. 
    \end{lemma}
    
    \begin{proof}
        See \S \ref{appendix:pf:decomposition2} for a detailed proof. 
    \end{proof}

\begin{remark}
    Similar regret decomposition results appear in the single-agent RL literature \citep{cai2020provably,efroni2020optimistic,yang2020bridging}, and they can be regarded as the special case of Lemma \ref{lemma:decomposition2}.
    Moreover, our regret decomposition lemma is independent of the myopic followers assumption, and thus can be applied to more general settings. 
\end{remark}

Lemma \ref{lemma:decomposition2} states that the regret has three sources: 
(i) computational error, which represents the convergence of the algorithm with the known model, (ii) statistical error, that is, the error caused by the inaccurate estimation of the model, and (iii) randomness, as aforementioned, which comes from executing random policies and interaction with random environment. 

Returning to the main proof, we only need to characterize these three types of errors, respectively. We first characterize the computational error by the following lemma.
    
    \begin{lemma}[Optimization Error] \label{lemma:myopic2}
        It holds that 
        \$
        &\sum_{k = 1}^K\sum_{h=1}^H\EE_{\pi^*, \nu^*}[\la Q_{h}^k(x_h^k, \cdot, \cdot), \pi_h^*(\cdot \,|\, x_h^k) \times \nu_h^*(\cdot \,|\, x_h^k) -  \pi_h^k(\cdot \,|\, x_h^k) \times \nu_h^k(\cdot \,|\, x_h^k) \ra] \le \epsilon KH.
        \$
    \end{lemma}
    
    \begin{proof}
        See \S \ref{appendix:pf:myopic2} for a detailed proof.
    \end{proof}

    Next, we establish an upper bound for the statistical error. Due to the uncertainty that arises from only observing limited data, the model prediction errors can be possibly large for the triple $(x, a, b)$ that are less visited or even unseen. Fortunately, however, we have the following lemma which characterizes the model prediction errors defined in \eqref{eq:def:td:error2}.
    
    \begin{lemma}[Optimism] \label{lemma:ucb2}
        It holds with probability at least $1 - p/2$ that 
        \$
        -2\min\{H, \Gamma_{h}^k(x, a, b)\} \le \delta_{h}^k(x, a, b) \le 0 
        \$
        for any $(k, h) \in [K] \times [H]$ and $(x, a, b) \in \cS \times \cA_l \times \cA_{f}$.
    \end{lemma}
    
    \begin{proof}
        See \S \ref{appendix:pf:ucb2} a detailed proof. 
    \end{proof}

    Lemma \ref{lemma:ucb2} states that $\delta_{h}^k(x, a, b) \le 0$ for any $(x, a, b) \in \cS \times \cA \times \cA$. Combining the definition of model prediction error in \eqref{eq:def:td:error2}, we obtain  
    \$
    Q_{h}^k(x, a, b) \ge r_{l, h}(x, a, b) + (\PP_hV_{h+1}^k)(x, a, b),
    \$
    which further implies that the estimated Q-function $Q_{\star, h}^k$ is “optimistic in the face of uncertainty.'' Moreover, Lemma \ref{lemma:ucb2} implies that $- \delta_{h}^k(x, a, b) \le 2\min\{H, \Gamma_{h}^k(x, a, b)\}$. Thus we only need to establish an upper bound for $2\sum_{k = 1}^K\sum_{h = 1}^H \min\{H, \Gamma_{h}^k (x_h^k, a_h^k, b_h^k)\}$, which is the total price paid for the optimism. As shown in the following lemma, we can derive an upper bound for this term by the elliptical potential lemma \citep{abbasi2011improved}. 
    
    \begin{lemma} \label{lemma:elliptical2}
        For the bonus function $\Gamma_{h}^k$ defined in Line \ref{line:def:bonus} of Algorithm \ref{alg2}, it holds that
        \$
        2\sum_{k = 1}^K\sum_{h = 1}^H \min\{H, \Gamma_{h}^k (x_h^k, a_h^k, b_h^k)\} \le \cO(\sqrt{d^3H^3T \iota^2}).
        \$
        Here $p \in (0, 1)$ and $\iota = \log(2dT/p)$ are defined in Theorem \ref{thm:multi:follower}.
    \end{lemma} 
    
    \begin{proof}
        See \S \ref{appendix:pf:elliptical2} for a detailed proof. 
    \end{proof}
    
    It remains to analyze the randomness, which is the purpose of the following lemma.

    \begin{lemma} \label{lemma:martingale2}
        For the $\zeta_{k, h}^1$ and $\zeta_{k, h}^2$ defined in Lemma \ref{lemma:decomposition2} and any $p \in (0, 1)$,  it holds with probability at least $1 - p/2$ that 
        \$
        \sum_{k = 1}^K\sum_{h = 1}^H (\zeta_{k, h}^1 + \zeta_{k, h}^2) \le \sqrt{16KH^3 \cdot \log(4/p)}.
        \$
    \end{lemma}
    
    \begin{proof}
        See \S \ref{appendix:pf:martingale2} for a detailed proof.
    \end{proof}
    Putting above lemmas together, we obtain 
    \# \label{eq:1111}
    \text{Regret}(K) \le \cO(\sqrt{d^3H^3T\iota^2})
    \#
    with probability at least $1 - p$, which concludes the proof of Theorem \ref{thm:multi:follower}.
    \end{proof}


\subsection{Proof of Lemma \ref{lemma:decomposition2}} \label{appendix:pf:decomposition2}

First, we establish a more general regret decomposition lemma, which immediately implies Lemma \ref{lemma:decomposition2}.
\begin{lemma}[General Decomposition for One Episode] \label{lemma:decomposition:general}
    Fix $k \in [K]$. Suppose $(\pi^k, \nu^k = \{\nu_{f_i}^k\}_{i \in [N]})$ are the policies executed by the leader $l$ and the followers $\{f_i\}_{i \in [N]}$ in the $k$-th episode. Moreover, suppose that $Q_{\star, h}^k$ and $V_{\star, h}^k = \la Q_{\star, h}^k, \pi_h^k \times \nu_h^k \ra$ are the estimated Q-function and value function for any $\star \in \{l, f_1, \cdots, f_N\}$ at $h$-th step of $k$-th episode.
    Then, for any policies $(\pi, \nu = \{\nu_{f_i}\}_{i \in [N]})$ and $\star \in \{l, f_1, \cdots, f_N\}$, we have 
    \$
    & V_{\star, 1}^{\pi, \nu}(x_1^k) - V_{\star, 1}^{\pi^k, \nu^k}(x_1^k) \\
    &\qquad= \underbrace{\sum_{h=1}^H\EE_{\pi, \nu}[\la Q_{\star,h}^k(x_h^k, \cdot, \cdot), \pi_h(\cdot \,|\, x_h^k) \times \nu_h(\cdot \,|\, x_h^k) -  \pi_h^k(\cdot \,|\, x_h^k) \times \nu_h^k(\cdot \,|\, x_h^k) \ra]}_{\text{Computational Error}} \notag\\
    &  \qquad\qquad + \underbrace{\sum_{h=1}^H \bigl(\EE_{\pi, \nu}[\delta_{\star, h}^k(x_h, a_h, b_{h})] - \delta_{\star, h}^k(x_h^k, a_h^k, b_{h}^k) \bigr)}_{\text{Statistical Error}} + \underbrace{\sum_{h = 1}^H (\zeta^1_{\star, k, h} + \zeta^2_{\star, k, h})}_{\text{Randomness}} ,
    \$ 
    where $\la Q_{\star, h}^k, \pi_h^k \times \nu_h^k \ra = \la Q_{\star, h}^k, \pi_h^k \times \nu_{f_1, h}^k \times \cdots \times \nu_{f_N, h}^k\ra_{\cA_l \times \cA_{f}}$ and $\la Q_{h}^k(x_h^k, \cdot, \cdot), \pi_h^*(\cdot \,|\, x_h^k) \times \nu_h^*(\cdot \,|\, x_h^k) -  \pi_h^k(\cdot \,|\, x_h^k) \times \nu_h^k(\cdot \,|\, x_h^k) \ra = \la Q_{h}^k(x_h^k, \cdot, \cdot, \cdots, \cdot), \pi_h^*(\cdot \,|\, x_h^k) \times \nu_{f_1, h}^*(\cdot \,|\, x_h^k) \times \cdots \nu_{f_N, h}^*(\cdot \,|\, x_h^k) -  \pi_h^k(\cdot \,|\, x_h^k) \times \nu_{f_1, h}^k(\cdot \,|\, x_h^k) \times \cdots \nu_{f_N, h}^k(\cdot \,|\, x_h^k) \ra_{\cA_l \times \cA_{f}}$. 
    Here $\delta_{\star, h}^k$ is the model prediction error defined by 
    \# \label{eq:def:td:error3} 
    \delta_{\star, h}^k = r_{\star, h} + \PP_hV_{\star, h+1}^k - Q_{\star, h}^k,
    \#
    and $\zeta_{\star, k, h}^1$ and $\zeta_{\star, k, h}^2$ are defined by 
    {\small
    \begin{equation}
    \begin{aligned} \label{eq:def:martingale3}
    \zeta_{\star, k, h}^1 &= [V_{\star, h}^k(x_h^k) - V_{\star, h}^{\pi^k, \nu^k}(x_h^k)] - [Q_{\star, h}^k(x_h^k, a_h^k, b_{h}^k) - Q_{\star, h}^{\pi^k, \nu^k}(x_h^k, a_h^k, b_{h}^k)], \\
    \zeta_{\star, k, h}^2 &= [(\PP_h V_{\star, h+1}^k)(x_h^k, a_h^k, b_{h}^k) - (\PP_h V_{\star, h+1}^{\pi^k, \nu^k})(x_h^k, a_h^k, b_{h}^k)] - [V_{\star, h+1}^k(x_{h+1}^k) - V_{\star, h+1}^{\pi^k, \nu^k}(x_{h+1}^k)]. 
    \end{aligned}
    \end{equation}
    }
\end{lemma}

\begin{proof}[Proof of Lemma \ref{lemma:decomposition:general}]
    To facilitate our analysis, for any $\nu = \{\nu_{f_i}\}_{i \in [N]}$ and $(h, x) \in [H] \times \cS$, we denote $\nu_{f_1, h}(\cdot \,|\, x) \times \cdots \nu_{f_N, h}(\cdot \,|\, x)$ by $\nu_h(\cdot \,|\, x)$. Moreover, we define two operators $\JJ_h$ and $\JJ_{k, h}$ respectively by
        \begin{equation}
        \begin{aligned} \label{eq:61000}
        (\JJ_hf)(x) &= \la f(x, \cdot, \cdot), \pi_h(\cdot \,|\, x) \times \nu_h(\cdot \,|\, x) \ra, \\
        (\JJ_{k, h}f)(x) &= \la f(x, \cdot, \cdot), \pi_h^k(\cdot \,|\, x) \times \nu_h^k(\cdot \,|\, x) \ra
        \end{aligned}
        \end{equation}
        for any $h \in [H]$ and any function $f: \cS \times \cA_l \times \cA_{f}  \rightarrow \RR$. Also, we define 
        \# \label{eq:61001}
        \xi_{\star, h}^k(x) &= (\JJ_hQ_{\star, f}^k)(x) - (\JJ_{k, h}Q_{\star, f}^k)(x) \notag\\
        &= \la Q_{\star,h}^k(x, \cdot, \cdot), \pi_h(\cdot \,|\, x) \times \nu_h(\cdot \,|\, x) -  \pi_h^k(\cdot \,|\, x) \times \nu_h^k(\cdot \,|\, x) \ra
        \#
        for any $(h, x) \in  [H] \times \cS$ and $\star \in \{l, f_1, \cdots, f_N\}$. 

        Using the above notation, we decompose the regret at the $k$-th episode into the following two terms, 
        \# \label{eq:61002}
        V_{\star, 1}^{\pi, \nu}(x_1^k) - V_1^{\pi^k, \nu^k}(x_1^k) = \underbrace{V_{\star, 1}^{\pi, \nu}(x_1^k) - V_{\star, 1}^k(x_1^k)}_{\rm (i)} + \underbrace{V_{\star, 1}^k(x_1^k) - V_1^{\pi^k, \nu^k}(x_1^k)}_{\rm (ii)} .
        \#
        Then we characterize these two terms respectively.
        \vskip 4pt
        \noindent{\bf Term (i).}
        By the Bellman equation in \eqref{eq:bellman} and the definition of the operator $\JJ_h$ in \eqref{eq:61000}, we have $V_{\star, h}^{\pi, \nu} = \JJ_hQ_{\star, h}^{\pi, \nu}$. Similar, by the definition of $ V_{\star, h}^k$ and the definition of the operator $\JJ_{k, h}$ in \eqref{eq:61000}, we have $ V_{\star, h}^k = \JJ_{k, h}Q_{\star, h}^k$. Hence, for any $h \in  [H]$, we have
        \# \label{eq:61003}
        V_{\star, h}^{\pi, \nu} - V_{\star, h}^k &= \JJ_hQ_{\star, h}^{\pi, \nu} - \JJ_{k, h}Q_{\star, h}^k = (\JJ_hQ_{\star, h}^{\pi, \nu} - \JJ_{h}Q_{\star, h}^k) +  (\JJ_{h}Q_{\star, h}^k - \JJ_{k, h}Q_{\star, h}^k) \notag\\
        &= \JJ_h(Q_{\star, h}^{\pi, \nu} - Q_{\star, h}^k) + \xi_{\star, h}^k,
        \#
        where the last inequality is obtained by the fact that $\JJ_h$ is a linear operator and the definition of $\xi_{\star, h}^k$ in \eqref{eq:61001}. Meanwhile, by the Bellman equation in \eqref{eq:bellman} and the definition of the prediction error $\delta_{\star, h}^k$ in \eqref{eq:def:td:error2}, we obtain
        \# \label{eq:61004}
        Q_{\star, h}^{\pi, \nu} - Q_{\star, h}^k &= (r_{\star, h} + \PP_hV_{\star, h+1}^{\pi, \nu}) - (r_{\star, h} + \PP_hV_{\star, h+1}^k - \delta_{\star, h}^k) \notag\\
        &= \PP_h(V_{\star, h+1}^{\pi, \nu} - V_{\star, h+1}^k) + \delta_{\star, h}^k .
        \#
        Putting \eqref{eq:61003} and \eqref{eq:61004} together, we further obtain
        \# \label{eq:61005}
        V_{\star, h}^{\pi, \nu} - V_{\star, h}^k = \JJ_h\PP_h(V_{\star, h+1}^{\pi, \nu} - V_{\star, h+1}^k) + \JJ_h\delta_{\star, h}^k + \xi_{\star, h}^k
        \#
        for any $h \in [H]$ and $\star \in \{l, f_1, \cdots, f_N\}$. By recursively applying \eqref{eq:61005} for all $h \in [H]$, we have 
        \# \label{eq:61006}
        V_{\star, 1}^{\pi, \nu} - V_{\star, 1}^k &= \Bigl(\prod_{h = 1}^H \JJ_h\PP_h\Bigr)(V_{\star, H+1}^{\pi, \nu} - V_{\star, H+1}^{\pi^k, \nu^k}) + \sum_{h = 1}^H \Bigl(\sum_{i = 1}^h \JJ_i\PP_i \Bigr)\JJ_h\delta_{\star, h}^k + \sum_{h = 1}^H \Bigl(\sum_{i = 1}^h \JJ_i\PP_i \Bigr)\xi_{\star, h}^k \notag\\
        & = \sum_{h = 1}^H \Bigl(\sum_{i = 1}^h \JJ_i\PP_i \Bigr)\JJ_h\delta_{\star, h}^k + \sum_{h = 1}^H \Bigl(\sum_{i = 1}^h \JJ_i\PP_i \Bigr)\xi_{\star, h}^k,
        \#
        where the last equality follows from the fact that $V_{\star, H+1}^{\pi, \nu} = V_{\star, H+1}^{\pi^k, \nu^k} = 0$. Thus, by utilizing the definition of $\xi_{\star, h}^k$ in \eqref{eq:61001}, we further obtain 
        \# \label{eq:61007}
        V_{\star, 1}^{\pi, \nu}(x_1^k) - V_{\star, 1}^k(x_1^k) &= 
        \EE_{\pi, \nu} \biggl[ \sum_{h = 1}^H \la Q_{\star,h}^k(x_h^k, \cdot, \cdot), \pi_h(\cdot \,|\, x_h^k) \times \nu_h(\cdot \,|\, x_h^k) -  \pi_h^k(\cdot \,|\, x_h^k) \times \nu_h^k(\cdot \,|\, x_h^k) \ra \biggr] \notag \\
        & \qquad \qquad + \EE_{\pi, \nu} \biggl[ \sum_{h = 1}^H \delta_{\star, h}^k(x_h, a_h, b_{h}) \biggr] 
        \#
        for any $k \in [K]$ and $\star \in \{l, f_1, \cdots, f_N\}$.

        \vskip 4pt
        \noindent{\bf Term (ii).}
        Recall that we denote $\{b_{f_i, h}^k\}_{i \in [N]}$ by $b_h^k$ for any $h \in [H]$. Then, for any $h \in [H]$ and $\star \in \{l, f_1, \cdots, f_N\}$, by the definition of model prediction error in \eqref{eq:def:td:error3}, we have 
        \# \label{eq:61008}
        \delta_{\star, h}^k(x_h^k, a_h^k, b_h^k) &= r_{\star, h}^k(x_h^k, a_h^k, b_h^k) + (\PP_hV_{\star, h+1}^k)(x_h^k, a_h^k, b_h^k) - Q_{\star, h}^k(x_h^k, a_h^k, b_h^k) \notag\\
        & = [r_{\star, h}^k(x_h^k, a_h^k, b_h^k) + (\PP_hV_{\star, h+1}^k)(x_h^k, a_h^k, b_h^k) - Q_{\star, h}^{\pi^k, \nu^k}(x_h^k, a_h^k, b_h^k)] \notag\\
        & \qquad + [Q_{\star, h}^{\pi^k, \nu^k}(x_h^k, a_h^k, b_h^k) - Q_{\star, h}^k(x_h^k, a_h^k, b_h^k)] \notag \\
        &= (\PP_hV_{\star, h+1}^k - \PP_hV_{\star, h+1}^{\pi^k, \nu^k})(x_h^k, a_h^k, b_h^k) + (Q_{\star, h}^{\pi^k, \nu^k} - Q_{\star, h}^k)(x_h^k, a_h^k, b_h^k),
        \#
        where the last equation is obtained by the Bellman equation in \eqref{eq:bellman}. 
        Thus, by \eqref{eq:61008}, we have
        \# \label{eq:61009}
        &V_{\star, h}^k(x_h^k) - V_{\star, h}^{\pi^k, \nu^k}(x_h^k) \notag\\
        & \qquad = V_{\star, h}^k(x_h^k) - V_{\star, h}^{\pi^k, \nu^k}(x_h^k) + (Q_{\star, h}^{\pi^k, \nu^k} - Q_{\star, h}^k)(x_h^k, a_h^k, b_h^k)  \notag\\
        &\qquad\qquad + (\PP_hV_{\star, h+1}^k - \PP_hV_{\star, h+1}^{\pi^k, \nu^k})(x_h^k, a_h^k, b_h^k) - \delta_{\star, h}^k(x_h^k, a_h^k, b_h^k) \notag\\
        & \qquad = V_{\star, h}^k(x_h^k) - V_{\star, h}^{\pi^k, \nu^k}(x_h^k) - (Q_{\star, h}^k - Q_{\star, h}^{\pi^k, \nu^k})(x_h^k, a_h^k, b_h^k)  \notag\\
        & \qquad\qquad + \bigl(\PP_h(V_{\star, h+1}^k - V_{\star, h+1}^{\pi^k, \nu^k})\bigr)(x_h^k, a_h^k, b_h^k) - (V_{\star, h+1}^k - V_{\star, h+1}^{\pi^k, \nu^k})(x_h^k) \notag\\
        &\qquad\qquad + (V_{\star, h+1}^k - V_{\star, h+1}^{\pi^k, \nu^k})(x_h^k)  - \delta_{\star, h}^k(x_h^k, a_h^k, b_h^k)
        \#
        for any $h \in [H]$ and $\star \in \{l, f_1, \cdots, f_N\}$. By the definitions of $\zeta^1_{\star, k, h}$ and $\zeta^2_{\star, k, h}$ in \eqref{eq:def:martingale3}, \eqref{eq:61009} can be written as
        \# \label{eq:61010}
        V_{\star, h}^k(x_h^k) - V_{\star, h}^{\pi^k, \nu^k}(x_h^k) = [V_{\star, h+1}^k(x_h^k) - V_{\star, h+1}^{\pi^k, \nu^k}(x_h^k)] + \zeta^1_{\star, k, h} + \zeta^2_{\star, k, h}  - \delta_{\star, h}^k(x_h^k, a_h^k, b_h^k).
        \#
        For any $\star \in \{l, f_1, \cdots, f_N\}$, recursively expanding \eqref{eq:61010} across $h \in [H]$ yields 
        \# \label{eq:61011}
        &V_{\star, 1}^k(x_1^k) - V_{\star, 1}^{\pi^k, \nu^k}(x_1^k) \notag\\
        & \qquad = V_{\star, H+1}^k(x_{H+1}^k) - V_{\star, H+1}^{\pi^k, \nu^k}(x_{H+1}^k) + \sum_{h = 1}^H(\zeta^1_{\star, k, h} + \zeta^2_{\star, k, h})  - \sum_{h = 1}^H \delta_{\star, h}^k(x_h^k, a_h^k, b_h^k) \notag\\
        & \qquad = \sum_{h = 1}^H(\zeta^1_{\star, k, h} + \zeta^2_{\star, k, h})  - \sum_{h = 1}^H \delta_{\star, h}^k(x_h^k, a_h^k, b_h^k),
        \#
        where the last equality follows from the fact that $V_{\star, H+1}^k(x_{H+1}^k) = V_{\star, H+1}^{\pi^k, \nu^k}(x_{H+1}^k) = 0$. 

        Plugging \eqref{eq:61007} and \eqref{eq:61011} into \eqref{eq:61002}, we obtain
        \$
        & V_{\star, 1}^{\pi, \nu}(x_1^k) - V_{\star, 1}^{\pi^k, \nu^k}(x_1^k) \notag \\
        & \qquad = \underbrace{\sum_{h=1}^H\EE_{\pi, \nu}[\la Q_{\star,h}^k(x_h^k, \cdot, \cdot), \pi_h(\cdot \,|\, x_h^k) \times \nu_h(\cdot \,|\, x_h^k) -  \pi_h^k(\cdot \,|\, x_h^k) \times \nu_h^k(\cdot \,|\, x_h^k) \ra]}_{\text{Computational Error}} \notag\\
        & \qquad\qquad + \underbrace{\sum_{h=1}^H \bigl(\EE_{\pi, \nu}[\delta_{\star, h}^k(x_h, a_h, b_h)] - \delta_{\star, h}^k(x_h^k, a_h^k, b_h^k) \bigr)}_{\text{Statistical Error}} + \underbrace{\sum_{h = 1}^H (\zeta^1_{\star, k, h} + \zeta^2_{\star, k, h})}_{\text{Randomness}} 
        \$
        for any $(\pi, \nu)$ and $\star \in \{l, f_1, \cdots, f_N\}$. Therefore, we conclude the proof of Lemma \ref{lemma:decomposition2}.
\end{proof}

\begin{proof}[Proof of Lemma \ref{lemma:decomposition2}]
    For any $k \in [K]$, applying Lemma \ref{lemma:decomposition:general} with  $(\pi, \nu) = (\pi^*, \nu^*)$, we obtain
    \$
    & V_{l, 1}^{\pi^*, \nu^*}(x_1^k) - V_{l, 1}^{\pi^k, \nu^k}(x_1^k) \notag \\
    &\qquad = \sum_{h=1}^H\EE_{\pi^*, \nu^*}[\la Q_{h}^k(x_h^k, \cdot, \cdot), \pi_h^*(\cdot \,|\, x_h^k) \times \nu_h^*(\cdot \,|\, x_h^k) -  \pi_h^k(\cdot \,|\, x_h^k) \times \nu_h^k(\cdot \,|\, x_h^k) \ra]\\
    & \qquad\qquad + \sum_{h=1}^H \bigl(\EE_{\pi^*, \nu^*}[\delta_{h}^k(x_h, a_h, b_{h})] - \delta_{h}^k(x_h^k, a_h^k, b_{h}^k) \bigr) + \sum_{h = 1}^H (\zeta_{k, h}^1 + \zeta_{k, h}^2).
    \$
    Taking a summation over $k \in [K]$, we decompose \eqref{eq:320001} as desired, which concludes the proof of Lemma \ref{lemma:decomposition2}. 

\end{proof}

\subsection{Proof of Lemma \ref{lemma:myopic2}} \label{appendix:pf:myopic2}

\begin{proof}[Proof of Lemma \ref{lemma:myopic2}]
    By the myopic followers assumption, we have that, for the matrix game with payoff matrices $(\tilde{Q}(x_h^k, \cdot, \cdot), \{r_{f_i, h}^k(x_h^k, \cdot, \cdot)\}_{i \in [N]})$, $\nu_h^*(\cdot \,|\, x_h^k)$ belongs to the best response set of $\pi_h^*(\cdot \,|\, x_h^k)$. Moreover, we define $\tilde{\nu}_h^*(\cdot \,|\, x_h^k)$ as the policy belongs to the best response set of $\pi_h^*(\cdot \,|\, x_h^k)$ and breaks ties in favor of the leader.

    Recall that $(\pi_h^k(\cdot \,|\, x_h^k), \nu_h^k(\cdot \,|\, x_h^k) = \{\nu_{f_i, h}^k(\cdot \,|\, x_h^k)\}_{i \in [N]})$ is the Stackelberg-Nash equilibrium of the matrix game with payoff matrices $(\tilde{Q}(x_h^k, \cdot, \cdot, \cdot), \{r_{f_i, h}^k(x_h^k, \cdot, \cdot, \cdot)\}_{i \in [N]})$, which implies that $\pi_h^k(\cdot \,|\, x_h^k)$ is the ``best response to the best response,'' which further implies that  
    \# \label{eq:63005}
    \la \tilde{Q}(x_h^k, \cdot, \cdot), \pi_h^*(\cdot \,|\, x_h^k) \times \tilde{\nu}_h^*(\cdot \,|\, x_h^k) -  \pi_h^k(\cdot \,|\, x_h^k) \times \nu_h^k(\cdot \,|\, x_h^k) \ra \le 0
    \#
    for any $(k, h) \in [K] \times [H]$. Thus, for any $(k, h) \in [K] \times [H]$, we have 
    \# \label{eq:63006}
    &\la {Q}_h^k(x_h^k, \cdot, \cdot), \pi_h^*(\cdot \,|\, x_h^k) \times \nu_h^*(\cdot \,|\, x_h^k) -  \pi_h^k(\cdot \,|\, x_h^k) \times \nu_h^k(\cdot \,|\, x_h^k) \ra \notag\\
    & \qquad = \la \tilde{Q}(x_h^k, \cdot, \cdot), \pi_h^*(\cdot \,|\, x_h^k) \times \nu_h^*(\cdot \,|\, x_h^k) -  \pi_h^k(\cdot \,|\, x_h^k) \times \nu_h^k(\cdot \,|\, x_h^k) \ra \notag\\
    & \qquad \qquad + \la Q_h^k(x_h^k, \cdot, \cdot) - \tilde{Q}(x_h^k, \cdot, \cdot), \pi_h^*(\cdot \,|\, x_h^k) \times \nu_h^*(\cdot \,|\, x_h^k) -  \pi_h^k(\cdot \,|\, x_h^k) \times \nu_h^k(\cdot \,|\, x_h^k) \ra\\
    & \qquad \le \la \tilde{Q}(x_h^k, \cdot, \cdot), \pi_h^*(\cdot \,|\, x_h^k) \times \tilde{\nu}_h^*(\cdot \,|\, x_h^k) -  \pi_h^k(\cdot \,|\, x_h^k) \times \nu_h^k(\cdot \,|\, x_h^k) \ra \notag\\
    & \qquad \qquad + \la Q_h^k(x_h^k, \cdot, \cdot) - \tilde{Q}(x_h^k, \cdot, \cdot), \pi_h^*(\cdot \,|\, x_h^k) \times \nu_h^*(\cdot \,|\, x_h^k) -  \pi_h^k(\cdot \,|\, x_h^k) \times \nu_h^k(\cdot \,|\, x_h^k) \ra \notag\\
    & \qquad \le \epsilon,
    \#
    where the first inequality follows from the definition of $\tilde{\nu}_h^k(\cdot \,|\, x_h^k)$ and the last inequality uses~\eqref{eq:63005} and the fact that $\|Q_h^k - \tilde{Q}\|_{\infty} \le \epsilon$. By taking summation over $(k, h) \in [K] \times [H]$, we conclude the proof of Lemma \ref{lemma:myopic2}.

\end{proof}

\subsection{Proof of Lemma \ref{lemma:ucb2}} \label{appendix:pf:ucb2}

\begin{proof}[Proof of Lemma \ref{lemma:ucb2}] 
    Recall that the estimated Q-function $Q_{h}^k$ defined in Line \ref{line:def:q} of Algorithm \ref{alg2} takes the following form:
        \begin{equation}
        \begin{aligned} \label{eq:54000}
            &Q_{h}^k(\cdot, \cdot, \cdot) \leftarrow r_{l, h}(\cdot, \cdot, \cdot) + \Pi_{H - h}\{\phi(\cdot, \cdot, \cdot)^\top w_{h}^k + \Gamma_{h}^k(\cdot, \cdot, \cdot)\}, \\
            & \text{where } w_{h}^k = (\Lambda_{h}^k)^{-1}\Bigl(\sum_{\tau = 1}^{k - 1}\phi(x_h^\tau, a_h^\tau, b_h^\tau) \cdot V_{h+1}^k(x_{h+1}^\tau)\Bigr). 
        \end{aligned}
        \end{equation}
        Here $\Lambda_{h}^k$ and $\Gamma_{h}^k$ are defined in Lines \ref{line:def:lambda} and \ref{line:def:bonus} of Algorithm \ref{alg2}, respectively. Meanwhile, by Assumption \ref{assumption:linear}, we have
        \# \label{eq:54001}
        (\PP_hV_{h+1}^k)(x, a, b) &= \phi(x, a, b)^\top \la \mu_h, V_{h + 1}^k\ra  \notag\\
        &  = \phi(x, a, b)^\top (\Lambda_{h}^k)^{-1} \Lambda_{h}^k \la \mu_h, V_{h + 1}^k\ra 
        \# 
        for any $(k, h, x, a, b) \in [K] \times [H] \times \cS \times \cA_l \times \cA_{f}$. Here $\la \mu_h, V_{h + 1}^k\ra = \int_{\cS} V_{h+1}^k(x') \text{d}\mu_h(x')$. Together with the definition of $\Lambda_{h}^k$ in Line \ref{line:def:lambda} of Algorithm \ref{alg2}, we further obtain
        \# \label{eq:54002}
        (\PP_hV_{h+1}^k)(x, a, b) &= \phi(x, a, b)^\top (\Lambda_{h}^k)^{-1}  \Bigl(\sum_{\tau = 1}^{k - 1} \phi(x_h^\tau, a_h^\tau, b_h^\tau) \phi(x_h^\tau, a_h^\tau, b_h^\tau)^\top \la \mu_h, V_{h + 1}^k\ra +  \la \mu_h, V_{h + 1}^k\ra \Bigr) \notag\\
        & = \phi(x, a, b)^\top (\Lambda_{h}^k)^{-1}  \Bigl(\sum_{\tau = 1}^{k - 1} \phi(x_h^\tau, a_h^\tau, b_h^\tau) \cdot (\PP_hV_{h+1}^k)(x_h^\tau, a_h^\tau, b_h^\tau) +  \la \mu_h, V_{h + 1}^k\ra \Bigr),
        \#
        for any $(k, h, x, a, b) \in [K] \times [H] \times \cS \times \cA_l \times \cA_{f}$. Here the last equality uses \eqref{eq:54001}. Putting \eqref{eq:54000} and \eqref{eq:54002} together, we have 
        \# \label{eq:54003}
        &\phi(x, a, b)^\top w_{h}^k - (\PP_hV_{h+1}^k)(x, a, b)  \notag\\
        & \qquad = \underbrace{\phi(x, a, b)^\top (\Lambda_{h}^k)^{-1}  \Bigl(\sum_{\tau = 1}^{k - 1} \phi(x_h^\tau, a_h^\tau, b_h^\tau) \cdot \bigl( V_{h+1}^k(x_{h+1}^\tau) - (\PP_hV_{h+1}^k)(x_h^\tau, a_h^\tau, b_h^\tau) \bigr) \Bigr)}_{\rm (i)} \\
        &\qquad \qquad -  \underbrace{\phi(x, a, b)^\top (\Lambda_{h}^k)^{-1} \la \mu_h, V_{h + 1}^k\ra}_{\rm (ii)},  \notag
        \#
        for any $(k, h, x, a, b) \in [K] \times [H] \times \cS \times \cA_l \times \cA_{f}$. Next we upper bound these two terms.
        \vskip 4pt
        \noindent{\bf Term (i).}
        By the Cauchy-Schwarz inequality, we have
        \# \label{eq:54004}
        |{\rm (i)}| \le \|\phi(x, a, b)\|_{(\Lambda_{h}^k)^{-1}} \cdot \Big\|\sum_{\tau = 1}^{k - 1} \phi(x_h^\tau, a_h^\tau, b_h^\tau) \cdot \bigl( V_{h+1}^k(x_{h+1}^\tau) - (\PP_hV_{h+1}^k)(x_h^\tau, a_h^\tau, b_{h}^\tau) \bigr) \Big\|_{(\Lambda_{h}^k)^{-1}} 
        \#
        for any $(k, h, x, a) \in [K] \times [H] \times \cS \times \cA_l$. Under the event $\mathcal{E}$ defined in Lemma \ref{lemma:event}, we further have 
        \# \label{eq:540040}
        |{\rm (i)}| \le C' dH\sqrt{\log(2dT/p)} \cdot \|\phi(x, a)\|_{(\Lambda_{h}^k)^{-1}} 
        \#
        for any $(k, h, x, a) \in [K] \times [H] \times \cS \times \cA_l$. 
        \vskip 4pt
        \noindent{\bf Term (ii).}
        Similarly, by the Cauchy-Schwarz inequality, we obtain 
        \# \label{eq:54005}
        |{\rm (ii)}| &\le \|\phi(x, a, b)\|_{(\Lambda_{h}^k)^{-1}} \cdot \|\la \mu_h, V_{h + 1}^k\ra\|_{(\Lambda_{h}^k)^{-1}}  \notag\\
        & \le \|\phi(x, a, b)\|_{(\Lambda_{h}^k)^{-1}} \cdot \|\la \mu_h, V_{h + 1}^k\ra\|_{2} \le \sqrt{d}H \cdot \|\phi(x, a, b)\|_{(\Lambda_{h}^k)^{-1}},
        \#
        for any $(k, h, x, a, b) \in [K] \times [H] \times \cS \times \cA_l \times \cA_f$. Here the second inequality follows from the fact that $\Lambda_{h}^k \succeq I$ and the last inequality is obtained by 
        \$
        \|\la \mu_h, V_{h + 1}^k\ra\|_2 \le \|\mu_h(\cS)\|_2 \cdot \|V_{h+1}^k\|_{\infty} \le \sqrt{d}H.
        \$
        Here we use the fact that $\|V_{h+1}^k\|_{\infty} \le H$ and  Assumption \ref{assumption:linear}, which assumes $\|\mu_h(\cS)\|_2 \le \sqrt{d}$.
        Plugging \eqref{eq:540040} and \eqref{eq:54005} into \eqref{eq:54003}, we obtain that 
        \# \label{eq:54006}
        |\phi(x, a, b)^\top w_{h}^k - (\PP_hV_{h + 1}^k)(x, a, b)| \le CdH \sqrt{\log(2dT/p)} \cdot \|\phi(x, a, b)\|_{(\Lambda_{h}^k)^{-1}},
        \#
        for any $(k, h, x, a, b) \in [K] \times [H] \times \cS \times \cA_l \times \cA_{f}$ under the event $\mathcal{E}$. Here $C > 0$ is a constant. By setting 
        \# \label{eq:54007}
        \beta = CdH \sqrt{\log(2dT/p)} 
        \#
        in Line \ref{line:def:bonus} of Algorithm \ref{alg2}, \eqref{eq:54006} gives 
        \# \label{eq:54008}
        |\phi(x, a, b)^\top w_{h}^k - (\PP_hV_{h + 1}^k)(x, a, b)| \le \Gamma_{h}^k(x, a, b)
        \#
        for any $(k, h, x, a, b) \in [K] \times [H] \times \cS \times \cA_l \times \cA_{f}$ under the event $\mathcal{E}$.
        Meanwhile, by the truncation in Line \ref{line:def:q} of Algorithm \ref{alg2} and the fact that $r_{l, h} \in [-1, 1]$, we have $Q_{h}^k \in [-(H - h + 1), H - h + 1]$, which further implies that 
        \# \label{eq:54009}
        V_{h}^k \in [-(H - h + 1), H - h + 1],
        \#
        for any $(k, h) \in [K] \times [H]$. Hence, by \eqref{eq:54008}, we have 
        \# \label{eq:54010}
        \phi(x, a, b)^\top w_{h}^k + \Gamma_{h}^k(x, a, b) \ge (\PP_hV_{h + 1}^k)(x, a, b) \ge H - h,
        \# 
        for any $(k, h, x, a, b) \in [K] \times [H] \times \cS \times \cA_l \times \cA_{f}$ under the event $\mathcal{E}$,
        where the last inequality is obtained by \eqref{eq:54009}.
        Thus, for the model prediction error defined in \eqref{eq:def:td:error2}, we have
        \# \label{eq:54011}
        - \delta_{h}^k(x, a, b) &= Q_{h}^k(x, a, b) - r_{l, h}(x, a, b) - \PP_hV_{h+1}^k(x, a, b) \notag\\
        & \le \phi(x, a, b)^\top w_{h}^k + \Gamma_{h}^k(x, a, b) - \PP_hV_{h+1}^k(x, a, b) \notag\\
        &\le 2\Gamma_{h}^k(x, a, b),
        \#
        for any $(k, h, x, a, b) \in [K] \times [H] \times \cS \times \cA_l \times \cA_{f}$ under the event $\mathcal{E}$. Moreover, by the definition of the model prediction error, we have $- \delta_{h}^k(\cdot, \cdot, \cdot) \le 2H$. Together with \eqref{eq:54011}, we have
        \# \label{eq:54012}
        - \delta_{h}^k(x, a, b) \le 2 \min\{H, \Gamma_{h}^k(x, a, b) \},
        \#
        for any $(k, h, x, a, b) \in [K] \times [H] \times \cS \times \cA_l \times \cA_{f}$ under the event $\mathcal{E}$.
        On the other hand, by \eqref{eq:evaluation}, we have
        \# \label{eq:54013}
        \delta_{h}^k(x, a, b) &= r_{l, h}(x, a, b) + \PP_hV_{h+1}^k(x, a, b) - Q_{h}^k(x, a, b) \notag\\
        & \le \PP_hV_{h+1}^k(x, a, b) - \min\{\phi(x, a, b)^\top w_{h}^k + \Gamma_{h}^k(x, a, b), H - h\} \notag\\
        & = \max\{\PP_hV_{h+1}^k(x, a, b) - \phi(x, a, b)^\top w_{h}^k - \Gamma_{h}^k(x, a, b),  \PP_hV_{h+1}^k(x, a, b) - (H - h)\} \notag\\
        & \le 0,
        \#
        for any $(k, h, x, a, b) \in [K] \times [H] \times \cS \times \cA_l \times \cA_{f}$ under the event $\mathcal{E}$.
        Here the last inequality follows from \eqref{eq:54008} and the fact that $V_{h+1}^k \le H - h$. Combining \eqref{eq:54012} and \eqref{eq:54013}, we conclude the proof of Lemma \ref{lemma:ucb2}. 
\end{proof}

\begin{lemma}\label{lemma:event}
    For any $p \in (0,1]$, the event $\mathcal{E}$ that, for any $(k,h)\in [K]\times[H]$,
     \$ 
     \Big\|\sum_{\tau = 1}^{k - 1} \phi(x_h^\tau, a_h^\tau, b_h^\tau) \cdot \bigl( V_{h+1}^k(x_{h+1}^\tau) - (\PP_hV_{h+1}^k)(x_h^\tau, a_h^\tau, b_h^\tau) \bigr) \Big\|_{(\Lambda_{h}^k)^{-1}} \le C'dH \sqrt{\log(2dT/p)}
     \$
     happens with probability at least $1- p/2$, where $C'>0$ is an absolute constant.
      \end{lemma}

     \begin{proof}[Proof of Lemma \ref{lemma:event}]
        Fix $(k, h) \in [K] \times [H]$. By Lemma \ref{lemma:bounded:weight1}, we have $w_{h+1}^k \le H\sqrt{dk}$, which implies that $Q_{h+1}^k \in \cQ_{h+1, \epsilon}^k$. Here $\cQ_{h+1, \epsilon}^k$ is defined in \eqref{eq:def:epsilon:net}. Moreover, as shown in Algorithm \ref{alg22}, we find a $\tilde{Q}$ in the $\epsilon$-net $\cQ_{h+1, \epsilon}^k$ such that $\|Q_{h+ 1}^k - \tilde{Q}\|_{\infty} \le \epsilon$. For any $x \in \cS$, let $(\tilde{\pi}(\cdot \,|\, x), \tilde{\nu} = \{\nu_{f_i}\}_{i = 1}^N)$ be the Stackelberg-Nash equilibrium of the matrix game with payoff matrices $(\tilde{Q}(x, \cdot, \cdot), \{r_{f_i, h+1}(x, \cdot, \cdot)\}_{i = 1}^N)$. Moreover, we define $\tilde{V}(x) = \EE_{a \sim \hat{\pi}(\cdot \,|\, x), b \sim \hat{\nu}(\cdot \,|\, x)}[\tilde{Q}(x, a, b)]$ for any $x \in \cS$. Then, we have
        \# \label{eq:490001}
        &\Big\|\sum_{\tau = 1}^{k - 1} \phi(x_h^\tau, a_h^\tau, b_h^\tau) \cdot \bigl( V_{h+1}^k(x_{h+1}^\tau) - (\PP_hV_{h+1}^k)(x_h^\tau, a_h^\tau, b_h^\tau) \bigr) \Big\|_{(\Lambda_{h}^k)^{-1}}  \notag\\
        & \qquad \le  \underbrace{\Big\|\sum_{\tau = 1}^{k - 1} \phi(x_h^\tau, a_h^\tau, b_h^\tau) \cdot \bigl( \tilde{V}(x_{h+1}^\tau) - (\PP_h\tilde{V})(x_h^\tau, a_h^\tau, b_h^\tau) \bigr) \Big\|_{(\Lambda_{h}^k)^{-1}}}_{\rm (i)} \\
        & \qquad \qquad + \underbrace{\Big\|\sum_{\tau = 1}^{k - 1} \phi(x_h^\tau, a_h^\tau, b_h^\tau) \cdot \bigl( [V_{h+1}^k(x_{h+1}^\tau) - \tilde{V}(x_{h+1}^\tau)] - \bigl(\PP_h(V_{h+1}^k - \tilde{V})\bigr)(x_h^\tau, a_h^\tau, b_h^\tau) \bigr) \Big\|_{(\Lambda_{h}^k)^{-1}}}_{\rm (ii)} . \notag
        \#
        By Lemma \ref{lemma:self:normalized} and a union bound argument, it holds for any $\tilde{Q} \in \cQ_{h + 1, \epsilon}^k$ with probability at least $1- p/2$ that 
        \# \label{eq:4900010}
        |{\rm (i)}| \le 4H^2 \Bigl(\frac{d}{2}\log(k+1) + \log \frac{2\cN_\epsilon}{p}\Bigr),
        \#
        where $\cN_\epsilon$ is the covering number of $Q_{h+1,\epsilon}$. Moreover, by applying Lemma \ref{lemma:covering} with $L = H\sqrt{dk}$ and $\lambda = 1$, \eqref{eq:4900010} gives that 
        \# \label{eq:490002}
        |{\rm (i)}| \le C'dH\sqrt{\log(dT/p)},
        \#
        with probability at least $1 - p/2$. Here $C'$ is a constant. Meanwhile, by the definition of $V_{h + 1}^k$ in Line \ref{line:v:l2} of Algorithm \ref{alg2}, we have ${V}_{h+1}^k(x) = \EE_{a \sim \hat{\pi}(\cdot \,|\, x), b \sim \hat{\nu}(\cdot \,|\, x)}[{Q}_{h+1}^k(x, a, b)]$, which yields that
        \$
        |V_{h + 1}^k(x) - \tilde{V}(x)| &= \big|\EE_{a \sim \hat{\pi}(\cdot \,|\, x), b \sim \hat{\nu}(\cdot \,|\, x)}[Q_{h+1}^k(x, a, b) - \tilde{Q}(x, a, b)] \big| \\
        & \le \EE_{a \sim \hat{\pi}(\cdot \,|\, x), b \sim \hat{\nu}(\cdot \,|\, x)} |Q_{h+1}^k(x, a, b) - \tilde{Q}(x, a, b) | \le \epsilon, 
        \$
        for any $x \in \cS$, which further implies that 
        \# \label{eq:490003}
        |{\rm (ii)}| \le \epsilon \cdot \sum_{\tau = 1}^{k - 1}\|\phi(x_h^\tau, a_h^\tau, b_h^\tau)\|_{(\Lambda_h^k)^{-1}} \le \epsilon k,
        \#
        where the last inequality follows from the fact that $\|\phi(\cdot, \cdot, \cdot)\|_{(\Lambda_h^k)^{-1}} \le \|\phi(\cdot, \cdot, \cdot)\|_2 \le 1$ for any $(k, h) \in [K] \times [H]$. Plugging \eqref{eq:490002} and \eqref{eq:490003} into \eqref{eq:490001}, together with the fact that $\epsilon = 1/KH$, we conclude the proof of Lemma \ref{lemma:event}.
     \end{proof}

    \begin{lemma}[Bounded Weight of Value Functions] \label{lemma:bounded:weight1}
        For all $(k, h) \in [K] \times [H]$, the linear coefficient $w_h^k$ defined in \eqref{eq:def:w} satisfies $\|w_h^k\| \le H\sqrt{kd}$.
    \end{lemma}

    \begin{proof}[Proof of Lemma \ref{lemma:bounded:weight1}]
        By the definition of $w_h^k$ in \eqref{eq:def:w} and the triangle inequality, we have 
        \# \label{eq:41001}
        \|w_h^k\| &= \Big\|(\Lambda_{h}^k)^{-1}\Bigl(\sum_{\tau = 1}^{k - 1}\phi(x_h^\tau, a_h^\tau, b_h^\tau) \cdot V_{h+1}^k(x_{h+1}^\tau)\Bigr) \Big\| \notag\\
        & \le \sum_{\tau = 1}^{k - 1}\|(\Lambda_{h}^k)^{-1}\phi(x_h^\tau, a_h^\tau, b_h^\tau) \cdot V_{h+1}^k(x_{h+1}^\tau) \| .
        \#
        Together with the fact that $|V_h^k(\cdot)| \le H$ for any $(k, h) \in [K] \times [H]$, \eqref{eq:41001} gives
        \# \label{eq:41002}
        \|w_h^k\| &\le H \cdot \sum_{\tau = 1}^{k - 1}\|(\Lambda_{h}^k)^{-1}\phi(x_h^\tau, a_h^\tau, b_h^\tau) \| \notag\\
        &\le H \cdot \sum_{\tau = 1}^{k - 1}\|(\Lambda_{h}^k)^{-1/2}\| \cdot \|\phi(x_h^\tau, a_h^\tau, b_h^\tau) \|_{(\Lambda_h^k)^{-1}} \notag\\
        & \le H \cdot \sum_{\tau = 1}^{k - 1}\|\phi(x_h^\tau, a_h^\tau, b_h^\tau) \|_{(\Lambda_h^k)^{-1}},
        \#
        where the second inequality uses the Cauchy-Schwarz inequality and the last inequality follows from the fact that $\Lambda_h^k \succeq I$ for any $(k, h) \in [K] \times [H]$. Then, by the Cauchy-Schwarz inequality, we obtain
        \# \label{eq:41003}
        \sum_{\tau = 1}^{k - 1}\|\phi(x_h^\tau, a_h^\tau, b_h^\tau) \|_{(\Lambda_h^k)^{-1}} &\le \sqrt{k} \cdot \Bigl(\sum_{\tau = 1}^{k - 1}\phi(x_h^\tau, a_h^\tau, b_h^\tau)^\top(\Lambda_h^k)^{-1} \phi(x_h^\tau, a_h^\tau, b_h^\tau)\Bigr)^{1/2} \notag\\
        & = \sqrt{k} \cdot  \Bigl(\sum_{\tau = 1}^{k - 1}\tr \big(\phi(x_h^\tau, a_h^\tau, b_h^\tau)^\top(\Lambda_h^k)^{-1} \phi(x_h^\tau, a_h^\tau, b_h^\tau)\big) \Bigr)^{1/2} \notag\\
        & = \sqrt{k} \cdot \Bigl(\tr \big((\Lambda_h^k)^{-1}\sum_{\tau = 1}^{k - 1} \phi(x_h^\tau, a_h^\tau, b_h^\tau)\phi(x_h^\tau, a_h^\tau, b_h^\tau)^\top\big) \Bigr)^{1/2}.
        \# 
        Finally, recall that $\Lambda_h^k = \sum_{\tau = 1}^{k - 1} \phi(x_h^\tau, a_h^\tau, b_h^\tau) \phi(x_h^\tau, a_h^\tau, b_h^\tau)^\top + I$, we have
        \# \label{eq:41004}
        \tr \Big((\Lambda_h^k)^{-1}\sum_{\tau = 1}^{k - 1} \phi(x_h^\tau, a_h^\tau, b_h^\tau)\phi(x_h^\tau, a_h^\tau, b_h^\tau)^\top\Big) \le \tr (I) = d.
        \#
        Plugging \eqref{eq:41003} and \eqref{eq:41004} into \eqref{eq:41002}, we conclude the proof of Lemma \ref{lemma:bounded:weight1}.
    \end{proof}

\subsection{Proof of Lemma \ref{lemma:elliptical2}} \label{appendix:pf:elliptical2}

\begin{proof}[Proof of Lemma \ref{lemma:elliptical2}]
    Recall the definition of $\Gamma_{h}^k$ in Line \ref{line:def:bonus} of Algorithm \ref{alg2}, we have 
    \# \label{eq:55000}
    2\sum_{k = 1}^K\sum_{h = 1}^H \min\{H, \Gamma_{h}^k (x_h^k, a_h^k, b_h^k)\} &= 2 \beta \cdot \sum_{k = 1}^K\sum_{h = 1}^H \min\{H/\beta, \|\phi(x_h^k, a_h^k, b_h^k)\|_{(\Lambda_{h}^k)^{-1}} \} \notag\\
    & \le 2 \beta \cdot \sum_{k = 1}^K\sum_{h = 1}^H \min\{1, \|\phi(x_h^k, a_h^k, b_h^k)\|_{(\Lambda_{h}^k)^{-1}} \}.
    \#
    Here the last inequality uses the fact that $\beta = CdH \sqrt{\log(2dT/p)}$, where $C > 1$ is a constant. By the Cauchy-Schwarz inequality, we further obtain that 
    \# \label{eq:55001}
    \sum_{k = 1}^K\sum_{h = 1}^H \min\{1, \|\phi(x_h^k, a_h^k, b_h^k)\|_{(\Lambda_{h}^k)^{-1}} \} &\le \sum_{h = 1}^H \Bigl(K \cdot \sum_{k = 1}^K \min\{1, \|\phi(x_h^k, a_h^k, b_h^k)\|^2_{(\Lambda_{h}^k)^{-1}}\} \Bigr) \notag\\
    & \le \sum_{h = 1}^H \sqrt{K} \cdot \biggl( 2\log\Big(\frac{\det(\Lambda_{h}^{K + 1})}{\det(\Lambda_{h}^1)} \Big) \biggr)^{1/2} ,
    \#
    where the last inequality follows from Lemma \ref{lemma:abbasi}. Moreover, Assumption \ref{assumption:linear} gives that 
    \$
    \|\phi(x, a, b)\|_2 \le 1,
    \$
    for any $(k, h, x, a, b) \in [K] \times [H] \times \cS \times \cA$, which further implies that 
    \# \label{eq:55002}
    \Lambda_{h}^{K + 1} = \sum_{k = 1}^K \phi(x_h^k, a_h^k, b_h^k)\phi(x_h^k, a_h^k, b_h^k)^\top + I \preceq (K + 1) \cdot I,
    \#
    for any $h \in [H]$. Combining \eqref{eq:55002} and the fact that $\Lambda_h^1 = I$, we obtain 
    \# \label{eq:55003}
    2\log\Big(\frac{\det(\Lambda_{h}^{K + 1})}{\det(\Lambda_{h}^1)} \Big) \le 2d \cdot \log(K + 1) \le 4d \cdot \log(K).
    \#
    Combining \eqref{eq:55000}, \eqref{eq:55001}, \eqref{eq:55002} and \eqref{eq:55003}, it holds that 
    \$
    2\sum_{k = 1}^K\sum_{h = 1}^H \min\{H, \Gamma_{h}^k (x_h^k, a_h^k, b_h^k)\} \le 2\beta \sqrt{dHT \cdot \log(K)} \le \cO(\sqrt{d^3H^3T\iota^2}),
    \$
    where $\iota = \log(2dT/p)$. Therefore, we conclude the proof of Lemma \ref{lemma:elliptical2}.
\end{proof}

\subsection{Proof of Lemma \ref{lemma:martingale2}} \label{appendix:pf:martingale2}

\begin{proof}[Proof of Lemma \ref{lemma:martingale2}]
    First, we show that $\{\zeta_{k, h}^1, \zeta_{k, h}^2\}_{(k, h) \in [K] \times [H]}$ can be written as a bounded martingale difference with respect to a filtration. Similar to \citet{cai2020provably}, we construct the following filtration. For any $(k, h) \in [K] \times [H]$, we define $\sigma$-algebras $\cF_{k, h}^1$ and $\cF_{k, h}^2$ as follows:
    \begin{equation}
    \begin{aligned} \label{eq:56000}
    &\cF_{k, h}^2 = \sigma \bigl(\{x_i^\tau, a_i^\tau, b_{1, i}^\tau, \cdots, b_{N, i}^\tau\}_{(\tau, i) \in [k - 1] \times [h]} \cup \{x_i^k, a_i^k, b_{1, i}^k, \cdots, b_{N, i}^k\}_{i \in [h]} \bigr), \\
    &\cF_{k, h}^2 = \sigma \bigl(\{x_i^\tau, a_i^\tau, b_{1, i}^\tau, \cdots, b_{N, i}^\tau\}_{(\tau, i) \in [k - 1] \times [h]} \cup \{x_i^k, a_i^k, b_{1, i}^k, \cdots, b_{N, i}^k\}_{i \in [h]} \cup \{x_{h + 1}^k\}\bigr),
    \end{aligned}
    \end{equation}
    where $x_{H + 1}$ is a null state for any $k \in [K]$. Here $\sigma(\cdot)$ denotes the $\sigma$-algebra generated by a finite set. Moreover, for any $(k, h, m) \in [K] \times [H] \times [2]$, we define the timestep index $t(k, h, m)$ as
    \# \label{eq:56001}
    t(k, h, m) = (k - 1) \cdot 2H + (h - 1) \cdot 2 + m.
    \#
    By the definitions of the $\sigma$-algebras in \eqref{eq:56000}, we have $\cF_{k, h}^m \subset \cF_{k', h'}^{m'}$ for any $t(k, h, m) \le t(k', m', h')$, which implies that the $\sigma$-algebra sequence $\{\cF_{k, h}^m\}_{(k, h, m) \in [K] \times [H] \times [2]}$ is a filtration. Moreover, by the definitions of $\{\zeta_{k, h}^1, \zeta_{k, h}^2\}_{(k, h) \in [K] \times [H]}$ in \eqref{eq:def:martingale2}, we have 
    \# \label{eq:56002}
    \zeta_{k, h}^1 \in \cF_{k, h}^1, \quad \zeta_{k, h}^2 \in \cF_{k, h}^2, \quad \EE[\zeta_{k, h}^1 \,|\, \cF_{k, h - 1}^2] = 0, \quad \EE[\zeta_{k, h}^2 \,|\, \cF_{k, h}^1] = 0,
    \#
    for any $(k, h) \in [K] \times [H]$. Here we identify $\cF_{k, 0}^2$ with $\cF_{k - 1, H}^2$ for any $k \ge 2$ and define $\cF_{1, 0, 2}$ be the empty set. Hence, we can define the martingale
    \# \label{eq:56003}
    \cM_{k, h}^m = \Big\{\sum_{k', h', m'}\zeta_{k', h'}^{m'} : t(k', h', m') \le t(k, h, m) \Big\} .
    \#
    Such a martingale is adapted to the filtration $\{\cF_{k, h}^m\}_{(k, h, m) \in [K] \times [H] \times [2]}$. In particular, we have 
    \# \label{eq:56004}
    \cM_{K, H}^2 = \sum_{k = 1}^K \sum_{h = 1}^H (\zeta_{k, h}^1 + \zeta_{k, h}^2).
    \#
    Moreover, note the fact that $V_{h}^k, Q_{h}^k, V_{l, h}^{\pi^k, \nu^k}, Q_{l, h}^{\pi^k, \nu^k} \in [-H, H]$, we further obtain 
    $|\zeta_{k, h}^m| \le 2H$, for any $(k, h, m) \in [K] \times [H] \times [2]$. Finally, by applying the Azuma-Hoeffding inequality to $\cM_{K, H}^2$ defined in \eqref{eq:56004}, we have
    \$
    \sum_{k = 1}^K\sum_{h = 1}^H (\zeta_{k, h}^1 + \zeta_{k, h}^2) \le \sqrt{16H^3K \cdot \log(4/p)},
    \$
    with probability at least $1 - p/2$, which concludes the proof of Lemma \ref{lemma:martingale2}.
\end{proof}

\section{Unknown Reward Setting} \label{appendix:unknown:reward}

We focus on the tabular case for simplicity, with the extension to the linear case  left as future work. We assume that $S = |\cS|$, $A_l = |\cA_l|$ and $A_f = |\cA_f| = |\cA_{f_1} \times \cdots \times \cA_{f_N}|$.
For simplicity, we use the shorthand $V_{\star}^{\pi, \nu} = V_{\star, 1}^{\pi, \nu}(x_1)$, where  $x_1 \in \cS$ is the fixed initial state.

\subsection{Algorithm}

We present the pseudocode of Reward-Free Explore algorithm \citep{jin2020reward} below.
\begin{algorithm}[H]
	\caption{Reward-Free Explore}
	\begin{algorithmic}[1]\label{alg5}
		\STATE {\bf Input:} iteration number $K_0$ and $K$.
		\STATE Let policy class $\Phi = \emptyset$. 
		\FOR{$(x, h) \in \cS \times [H]$}
		\STATE $r_{h'}(x', a', b') \leftarrow \mathbf{1} \left[x^{\prime}=x \text { and } h^{\prime}=h\right] \text { for all }\left(x^{\prime}, a^{\prime}, b', h^{\prime}\right) \in \mathcal{S} \times \mathcal{A}_l \times \cA_f \times[H]$.
		\STATE $\Phi^{(x, h)} \leftarrow \text{EULER} \left(r, K_{0}\right)$. \footnote{}
		\STATE $\pi_{h}(\cdot \,|\, x) \leftarrow \operatorname{Uniform}(\mathcal{A}_l)$ and $\nu_{h}(\cdot \,|\, x) \leftarrow \operatorname{Uniform}(\mathcal{A}_f)$ for all $(\pi, \nu) \in \Phi^{(x, h)}$. 
		\STATE $\Psi \leftarrow \Psi \cup \Phi^{(x, h)}$.
		\ENDFOR
		\FOR{$k = 1,\cdots, K$}
		\STATE Sample policy $(\pi, \nu) \sim \text{Uniform}(\Psi)$.
		\STATE Play the game $\cM$ using policy $\pi$ and $\nu$, and observe the trajectory $\{x_h^k, a_h^k, b_h^k\}_{h \in [H]}$ and rewards $\{r_{\star, h}(x_h^k, a_h^k, b_h^k)\}_{h \in [H]}$.
		\ENDFOR
		\STATE Calculate the empirical reward as
		\$
		\hat{r}_{\star, h}(x, a, b) = \frac{\sum_{k = 1}^K r_{\star, h}(x, a, b) \cdot \textbf{1}[x_h^k = x, a_h^k = a,x_{h + 1}^k = x']}{\sum_{k = 1}^K  \textbf{1}[x_h^k = x, a_h^k = a,x_{h + 1}^k = x']}.
		\$
	\end{algorithmic}
\end{algorithm} 
\footnotetext[2]{Here EULER is a single-agent RL algorithm proposed in \citet{zanette2019tighter}.}

\begin{lemma} \label{lemma:est:reward:error}
	Fix $\varepsilon, p > 0$. If we set $K_0 \ge \Omega(H^7S^4A_l/\varepsilon)$ and $K \ge \Omega(H^3S^2A_lA_f/\varepsilon^2)$ in Algorithm \ref{alg5}, then we have that the empirical rewards $\{\hat{r}_l, \hat{r}_{f_1}, \cdots \hat{r}_{f_N}\}$ and corresponding value functions $(\hat{V}_l, \hat{V}_{f_1}, \cdots, \hat{V}_{f_N})$ satisfy 
    \$
	\sup_{\pi, \nu}|\hat{V}_\star^{\pi, \nu} - V_{\star}^{\pi, \nu}| \le \varepsilon,
	\$
	with probability at least $1 - p$. Here $\Omega(\cdot)$ hides some logarithmic factors.
\end{lemma}

\begin{proof}
	This lemma is a simple extension of Lemma D.1 in \citet{bai2021sample}. They focus on the MDP setting and we consider the more complex Markov games. For completeness, we present the detailed proof in \S \ref{appendix:pf:est:reward:error}. 
\end{proof}

Lemma \ref{lemma:est:reward:error} states that we can obtain estimated reward functions and the associated value functions are an $\varepsilon$-approximation of the true value functions, which further implies that the SNE with respect to the estimated reward functions is a good approximation of the SNE in the original problem. We also remark that if we consider the Markov games with only one follower and aim to find the Stackelberg equilibria, we can provide a more refined analysis. See \S \ref{appendix:stackelberg} for more details.




\subsection{Proof of Lemma \ref{lemma:est:reward:error}} \label{appendix:pf:est:reward:error}

Before our proof, we present a useful lemma.
\begin{lemma} \label{lemma:significant}
	We define the set of $\delta$-significant states as 
	\# \label{eq:def:significant:set}
	\cS_h^\delta = \{s: \max_{\pi, \nu}\PP_h^{\pi, \nu}(x) \ge \delta\},
	\#
	where $\PP_h^{\pi,\nu}(x)$ is the probability of visiting $x$ at $h$-th step under policies $(\pi, \nu)$. Then we have 
	\$
	\max_{\pi, \nu} \frac{\PP_h^{\pi, \nu}(x)}{\frac{1}{K_0}\sum_{(\pi, \nu) \in \Phi^{(x, h)}}\PP_h^{\pi, \nu}(x)} \le 2
	\$
	for any $s \in \cS_h^\delta$. Here $\PP_h^\pi(x, a)$ is the probability of visiting $(x, a)$ at $h$-th step under policy $\pi$.
\end{lemma}

\begin{proof}
	See the proof of Theorem 3.3 in \citet{jin2020reward} for more details.
\end{proof}

Now, we are ready to prove Lemma \ref{lemma:est:reward:error}.

\begin{proof}[Proof of Lemma \ref{lemma:est:reward:error}]
	For any $(\pi, \nu)$, we denote $\PP_h^{\pi, \nu}(x, a, b)$ as the probability of visiting $(x, a, b)$ at $h$-th step under policies $(\pi, \nu)$. Under this notion, by Lemma \ref{lemma:significant} and the fact that all policies in $\Phi^{(x, h)}$ are uniform at $(x, h)$, we have 
	\$
	\max_{\pi, a, b} \frac{\PP_h^{\pi, \nu}(x, a, b)}{\frac{1}{K_0}\sum_{\pi \in \Phi^{(x, h)}}\PP_h^{\pi, \nu}(x, a, b)} \le 2A_lA_f,
	\$
	where $|\cA_l| = A_l$ and $A_f = |\cA_f| = |\cA_{f_1} \times \cdots \times \cA_{f_N}|$. 
    Thus, for any $\delta$-significant $(x, h)$, we have
	\$
	\max_{\pi, \nu, a, b} \frac{\PP_h^{\pi, \nu}(x, a, b)}{\frac{1}{K_0SH}\sum_{(\pi, \nu) \in \cup\{\Phi^{(x, h)}\}_{(x, h)}}\PP_h^{\pi, \nu}(x, a, b)} \le 2SA_lA_fH.
	\$
    Then the data obtained from Algorithm \ref{alg5} is sampled i.i.d. from some distribution $\zeta_h$, such that 
	\# \label{eq:sample:distribution}
	\max_{\pi, \nu, a, b} \frac{\PP_h^{\pi, \nu}(x, a, b)}{\zeta_h(x, a, b)} \le 2SA_lA_fH.
	\#
	for any $s \in \cS_h^\delta$. Back to our proof, we have
	\# \label{eq:5400}
	|\hat{V}_{\star}^{\pi, \nu} - {V}_{\star}^{\pi, \nu}| &= \biggl|\sum_{h = 1}^H \sum_{x, a, b}\PP_{h}^{\pi, \nu}(x, a, b) \cdot \big(\hat{r}_{\star, h}(x, a, b) - r_{\star, h}(x, a, b)\big)\biggr| \notag\\
	& = \biggl|\sum_{h = 1}^H \sum_{x, a, b}\PP_{h}^{\pi, \nu}(x, a, b) \cdot \big(\hat{r}_{\star, h}(x, a, b) - r_{\star, h}(x, a, b)\big)\biggr| \notag\\
	& \le \underbrace{\biggl|\sum_{h = 1}^H \sum_{x \notin \cS_h^\delta, a, b}\PP_{h}^{\pi, \nu}(x, a, b) \cdot \big(\hat{r}_{\star, h}(x, a, b) - r_{\star, h}(x, a, b)\big)\biggr|}_{\rm (i)} \notag\\
	& \qquad + \underbrace{\biggl|\sum_{h = 1}^H \sum_{x \in \cS_h^\delta, a, b}\PP_{h}^{\pi, \nu}(x, a, b) \cdot \big(\hat{r}_{\star, h}(x, a, b) - r_{\star, h}(x, a, b)\big)\biggr|}_{\rm (ii)}.
	\#
	Clearly,
	\# \label{eq:540}
	{\rm (i)} \le \sum_{h = 1}^H \sum_{x \notin \cS_h^\delta, a, b} \PP_h^{\pi, \nu}(s, a, b) = \sum_{h = 1}^H \sum_{x \notin \cS_h^\delta} \PP_h^{\pi}(x) \le HS\delta \le \varepsilon/2,
	\#
	where the second inequality uses the definition of $\delta$-significant set in \eqref{eq:def:significant:set} and the last inequality is implied by the fact that $\delta = \varepsilon/2H^2S$. Meanwhile, we have 
	\# \label{eq:541}
	{\rm (ii)} &\le \sum_{h = 1}^H \biggl| \sum_{x \in \cS_h^\delta, a, b}\PP_{h}^{\pi, \nu}(x, a, b) \cdot \big(\hat{r}_{\star, h}(x, a, b) - r_{\star, h}(x, a, b)\big)\biggr| \notag\\
	& \le \sum_{h = 1}^H \underbrace{\Big(\sum_{x \in \cS_h^\delta, a, b}\PP_{h}^{\pi, \nu}(x, a, b) \cdot \big(\hat{r}_{\star, h}(x, a, b) - r_{\star, h}(x, a, b)\big)^2 \Big)^{1/2}}_{\Delta_h}.
	\#
	Note that $\PP_h^{\pi,\nu}(x, a, b) = \PP_h^{\pi,\nu}(x) \cdot \pi_h(a \,|\, x) \cdot \nu_h(b \,|\, x)$, together with the Cauchy-Schwarz inequality, we further have 
	\# \label{eq:542}
	{\Delta_h} &\le  \max_{\pi' : \cS \rightarrow \cA_l, \nu' : \cS \rightarrow \cA_f}\Big(\sum_{x \in \cS_h^\delta, a, b}\PP_{h}^{\pi}(x) \cdot \big(\hat{r}_{\star, h}(x, a, b) - r_{\star, h}(x, a, b)\big)^2 \mathbf{1}[a = \pi'(s), b = \nu'(s) ] \big) \Big)^{1/2} \notag\\
	& \le  \max_{\pi' : \cS \rightarrow \cA_l, \nu' : \cS \rightarrow \cA_f}\Big(\sum_{x \in \cS_h^\delta, a, b}\PP_{h}^{\pi}(x) \cdot \big(\hat{r}_{\star, h}(x, a, b) - r_{\star, h}(x, a, b)\big)^2 \mathbf{1}[a = \pi'(s), b = \nu'(s) ] \big) \Big)^{1/2} \notag\\
	& \le \max_{\pi' : \cS \rightarrow \cA_l, \nu' : \cS \rightarrow \cA_f} {(2SA_lA_fH)}^{1/2} \notag\\
	& \quad \times \Big(\sum_{x \in \cS_h^\delta, a, b}\zeta_{h}(x, a, b) \cdot \big(\hat{r}_{\star, h}(x, a, b) - r_{\star, h}(x, a, b)\big)^2 \mathbf{1}[a = \pi'(s), b = \nu'(s) ] \big) \Big)^{1/2}, 
	\#
	where the last inequality follows from \eqref{eq:sample:distribution}. Moreover, by the Hoeffding inequality and a union bound for the reward estimates we have 
	\# \label{eq:543}
	&\Big(\sum_{x \in \cS_h^\delta, a, b}\zeta_{h}(x, a, b) \cdot \big(\hat{r}_{\star, h}(x, a, b) - r_{\star, h}(x, a, b)\big)^2 \mathbf{1}[a = \pi'(s), b = \nu'(s) ] \big) \Big)^{1/2} \notag\\
	& \qquad \le \Big(\sum_{x \in \cS_h^\delta, a, b}\zeta_{h}(x, a, b) \cdot \tilde{\cO}\big(\frac{1}{N_h(s, a, b)}\big) \mathbf{1}[a = \pi'(s), b = \nu'(s) ] \big) \Big)^{1/2}  .
	\#
	Choose $\delta = \varepsilon/2H^2S$. Together with \eqref{eq:sample:distribution}, we have $\zeta_h(s, a, b) \ge \varepsilon/4H^3S^2A_lA_f$ for any $s \in \cS_h^\delta$. Hence, we have $K \ge \Omega(H^3S^2A_lA_f/\varepsilon) \ge \Omega(1/\min_{s, a, b}\zeta_h(s, a, b))$. Applying a multiplicative Chernoff bound for the counter $N_h(s, a, b) \sim \text{Bin}(K, \zeta_h(s, a, b))$, we have  
	\# \label{eq:544}
	&\Big(\sum_{x \in \cS_h^\delta, a, b}\zeta_{h}(x, a, b) \cdot \tilde{\cO}\big(\frac{1}{N_h(s, a, b)}\big) \mathbf{1}[a = \pi'(s), b = \nu'(s) ] \big) \Big)^{1/2}  \notag \\
	&\qquad \le \Big(\sum_{x \in \cS_h^\delta, a, b}\zeta_{h}(x, a, b) \cdot \tilde{\cO}\big(\frac{1}{K\zeta_h(s, a, b)}\big) \mathbf{1}[a = \pi'(s), b = \nu'(s) ] \big) \Big)^{1/2}  \notag \\
	& \qquad = \tilde{\cO}\Big(\sqrt{\frac{S}{K}}\Big) .
	\#
	Plugging \eqref{eq:542}, \eqref{eq:543}, and \eqref{eq:543} into \eqref{eq:541}, we have 
	\# \label{eq:545}
	{\rm (ii)} \le \tilde{\cO}\Big( \sqrt{\frac{H^3S^2A_lA_f}{K}} \Big) \le \varepsilon/2,
	\#
	where the last inequality follows from our choice that $K \ge \Omega(H^3S^2A_lA_f/\varepsilon^2)$. Combining \eqref{eq:5400}, \eqref{eq:540} and \eqref{eq:545}, we have $|\hat{V}_{\star}^{\pi, \nu} - {V}_{\star}^{\pi, \nu}| \le \varepsilon$ for any $(\pi, \nu)$, which concludes the proof of Lemma~\ref{lemma:est:reward:error}.
\end{proof}

\section{Learning Stackelberg Equilibria} \label{appendix:stackelberg}
In this section, we analyze the sample efficiency of learning Stackelberg equilibria in two-player tabular Markov games without the known reward assumption. 

For simplicity, we use the shorthand $f = f_1$ and $V_{\star}^{\pi, \nu} = V_{\star, 1}^{\pi, \nu}(x_1)$, where  $x_1 \in \cS$ is the fixed initial state.
Meanwhile, for any $\varepsilon > 0$, we define the $\varepsilon$-approximate value of the best-case response by
\$
&V_{\varepsilon}^{\pi} = \max_{\nu \in \text{BR}_{\varepsilon}(\pi)} V_l^{\pi, \nu}, \\
&\text{BR}_{\varepsilon}(\pi) = \{\nu : V_{f}^{\pi, \nu} \ge \max_{\nu'}V_{f}^{\pi, \nu'} - \varepsilon\}.
\$
We immediately obtain that $\text{BR}(\pi) \subseteq \text{BR}_{\varepsilon}(\pi)$, which further implies $V_{\varepsilon}^{\pi} \ge V_l^{\pi, \nu^*(\pi)}$. Thus we can define the gap 
\# \label{eq:def:gap:vareps}
&\text{gap}_{\varepsilon} = \max_{\pi \in \Pi_\varepsilon}[V_{\varepsilon}^{\pi} - V_l^{\pi, \nu^*(\pi)}], \\
&\Pi_\varepsilon = \{\pi : V_\varepsilon^\pi \ge V^{\pi^*, \nu^*} - \varepsilon\}. \notag 
\#

\subsection{Algorithm}

As stated before, we first conduct a Reward-Free Explore algorithm (Algorithm~\ref{alg5}) to obtain the estimated rewards $(\hat{r}_l, \hat{r}_f)$. We also define $(\hat{V}_l, \hat{V}_f)$ as the corresponding value functions. Then we use Algorithm \ref{alg2} to solve the SNE with respect to the \emph{known} reward functions $(\hat{r}_l, \hat{r}_f)$. Specifically, we consider the following optimization problem of finding approximation Stackelberg equilibria with respect to the empirical rewards $(\hat{r}_l, \hat{r}_f)$:
\begin{equation}
\begin{aligned} \label{eq:problem:sne}
&\argmax_\pi \hat{V}_{3\varepsilon/4}(\pi) = \argmax_{\pi}  \hat{V}_l^{\pi, \nu(\pi)}, \\
& \nu(\pi) = \argmax_{\nu \in \hat{\text{BR}}_{3\varepsilon/4}(\pi)} \hat{V}_l^{\pi, \nu},\\
& \hat{\text{BR}}_{3\varepsilon/4}(\pi) = \big\{\nu: \hat{V}_{f}^{\pi, \nu} \geq \max_{\nu}\hat{V}_{f}^{\pi, \nu} - 3\varepsilon/4\big\}.
\end{aligned}
\end{equation}
Since $(\hat{r}_l, \hat{r}_f)$ are known to us, we can use Algorithm \ref{alg2} to obtain the solution $(\hat{\pi}, \hat{\nu} = \nu(\hat{\pi}))$, which is our approximate solution. See Algorithm \ref{Alg4} for more details.

\begin{algorithm}[H]
	\caption{Reward-Free Explore then Commit}
	\begin{algorithmic}[1] \label{Alg4}
		\STATE {\bf Input:} Accuracy coefficient $\varepsilon > 0$.
		\STATE Run the Reward-Free Explore algorithm (Algorithm~\ref{alg5}) with $K_0 \ge \Omega(H^7S^4A_l/\varepsilon)$ and $K \ge \Omega(H^3S^2A_lA_f/\varepsilon^2)$, and obtain empirical rewards $(\hat{r}_l, \hat{r}_f)$.
        \STATE Use Algorithm \ref{alg2} as an oracle to solve the problem defined in \eqref{eq:problem:sne} and obtain the solution $(\hat{\pi}, \hat{\nu} = \nu(\pi))$.
		\STATE {\bf Output:} $(\hat{\pi}, \hat{\nu})$.
	\end{algorithmic}
\end{algorithm}

\subsection{Theoretical Results}

The performance of Algorithm \ref{Alg4} is guaranteed by the following theorem.

\begin{theorem} \label{thm:stackelberg}
    Suppose Algorithm \ref{Alg4} outputs $(\hat{\pi}, \hat{\nu})$. Then it holds with probability at least $1 - p$ that  
	\$
	V_l^{\hat{\pi}, \nu^*(\hat{\pi})} \ge V_l^{\pi^*, \nu^*} - \text{gap}_{\varepsilon} - \varepsilon, \qquad V_{f}^{\hat{\pi}, \hat{\nu}} \ge V_{f}^{\hat{\pi}, \nu^*(\hat{\pi})} - \varepsilon .
	\$
\end{theorem}

\begin{proof}
    A similar analysis appears in \citet{bai2021sample}. As stated earlier, however, their setting is different from ours. For completeness, we provide a detailed proof here. 
    First, we show that 
    \# \label{eq:800}
	\text{BR}_{\varepsilon/2}(\pi) \subseteq \hat{\text{BR}}_{3\varepsilon/4}(\pi) \subseteq \text{BR}_{\varepsilon}(\pi).
	\#
	By choosing a large absolute constant in $K$, together with Lemma \ref{lemma:est:reward:error}, it holds for any $\star \in \{l, f\}$ that 
	\# \label{eq:801}
	\sup_{\pi, \nu} |\hat{V}_{\star}^{\pi, \nu} - V_{\star}^{\pi, \nu} | \le \varepsilon/8. 
	\#
    Meanwhile, for the empirical rewards $(\hat{r}_l, \hat{r}_f)$, we define the best response of leader's policy $\pi$ as $\hat{\nu^*(\pi)}$.
    Using this notation, for any $\nu \in \hat{\text{BR}}_{3\varepsilon/4}(\pi)$, we have
	\# \label{eq:802}
	&V_{f}^{\pi, \nu^*(\pi)} - V_{f}^{\pi, \nu} \notag\\
	&\qquad = \underbrace{(V_{f}^{\pi, \nu^*(\pi)} - \hat{V}_{f}^{\pi, \nu^*(\pi)})}_{\rm (i)} +  \underbrace{(\hat{V}_{f}^{\pi, \nu^*(\pi)} -  \hat{V}_{f}^{\pi, \hat{\nu^*(\pi)}})}_{\rm (ii)} + \underbrace{(\hat{V}_{f}^{\pi, \hat{\nu^*(\pi)}} - \hat{V}_{f}^{\pi, \nu})}_{\rm (iii)} + \underbrace{(\hat{V}_{f_i}^{\pi, \nu} - V_{f}^{\pi, \nu})}_{\rm (iv)} \notag\\
	& \qquad \le \varepsilon/8 + 0 + 3\varepsilon/4 + \varepsilon/8 \le \varepsilon.
	\#
    where ${\rm (i)} \le \varepsilon/8$ and ${\rm (iv)} \le \varepsilon/8$ is implied by the uniform convergence in \eqref{eq:801}, ${\rm (ii)} \le 0$ uses the definition of $\hat{\nu^*(\pi)}$, and ${\rm (iii)} \le 0$ follows from the fact that $\nu \in \hat{\text{BR}}_{3\varepsilon/4}(\pi)$.

    Similarly, for any $\nu \in \text{BR}_{\varepsilon/2}(\pi)$, we can show that 
	\# \label{eq:803}
	& \hat{V}_f^{\pi, \hat{\nu^*(\pi)}} - \hat{V}_f^{\pi, \nu} \notag\\
	& \qquad = (\hat{V}_f^{\pi, \hat{\nu^*(\pi)}} - V_f^{\pi, \hat{\nu^*(\pi)}}) + (V_f^{\pi, \hat{\nu^*(\pi)}} - V_f^{\pi, \nu^*(\pi)}) + (V_f^{\pi, \nu^*(\pi)} - V_f^{\pi, \nu}) + (V_f^{\pi, \nu} - \hat{V}_f^{\pi, \nu}) \notag\\
	& \qquad \le \varepsilon/8 + 0 + \varepsilon/2 + \varepsilon/8 = 3\varepsilon/4.
	\#
    Combining \eqref{eq:802} and \eqref{eq:803}, we obtain $\text{BR}_{\varepsilon/2}(\pi) \subseteq \hat{\text{BR}}_{3\varepsilon/4}(\pi) \subseteq \text{BR}_{\varepsilon}(\pi)$ as desired.

    Back to our proof, by the fact that $\hat{\pi}$ maximizes $\hat{V}_{3\varepsilon/4}^{\pi} = \max_{\nu \in \hat{\text{BR}}_{3\varepsilon/4}(\pi)}\hat{V}_l(\pi, \nu)$, we have
	\# \label{eq:804}
	\max_{\nu \in \hat{\text{BR}}_{3\varepsilon/4}(\hat{\pi})}\hat{V}_l^{\hat{\pi}, \nu} = \hat{V}_{3\varepsilon/4}^{\hat{\pi}} \ge \hat{V}_{3\varepsilon/4}^{\pi} = \max_{\nu \in \hat{\text{BR}}_{3\varepsilon/4}({\pi})}V_l^{{\pi}, \nu} \ge \max_{\nu \in \text{BR}_{\varepsilon/2}({\pi})}\hat{V}_l^{{\pi}, \nu}, 
	\#
    for any $\pi$. Here the last inequality uses the fact $\text{BR}_{\varepsilon/2}(\pi) \subseteq \hat{\text{BR}}_{3\varepsilon/4}(\pi)$ in \eqref{eq:800}.
	Together with the uniform convergence in \eqref{eq:801}, \eqref{eq:804} yields
	\# \label{eq:805}
	\max_{\nu \in \hat{\text{BR}}_{3\varepsilon/4}(\hat{\pi})}{V}_l^{\hat{\pi}, \nu} \ge \min_{\nu \in \text{BR}_{\varepsilon/2}({\pi})}{V}_l^{\pi, \nu} - \varepsilon/8 \ge V_{\varepsilon/2}^\pi - \varepsilon .
	\#
	Meanwhile, by the fact $\hat{\text{BR}}_{3\varepsilon/4}(\pi) \subseteq \text{BR}_{\varepsilon}(\pi)$ in \eqref{eq:801}, we have 
    \# \label{eq:806}
    V_{\varepsilon}^{\hat{\pi}} = \min_{\nu \in \text{BR}_{\varepsilon}(\hat{\pi})}{V}_l^{\hat{\pi}, \nu} \ge \min_{\nu \in \hat{\text{BR}}_{3\varepsilon/4}(\hat{\pi})}{V}_l^{\hat{\pi}, \nu}.
    \#
    Combining \eqref{eq:805} and \eqref{eq:806}, we have
	\# \label{eq:807}
	V_{\varepsilon}^{\hat{\pi}}  \ge \max_{\pi}V_{\varepsilon/2}^\pi - \varepsilon \ge \max_{\pi}V_l^{\pi, \nu^*(\pi)} - \varepsilon,
	\# 
    which implies that $\hat{\pi} \in \Pi_\varepsilon$. Furthermore, \eqref{eq:807} is equivalent to 
    \$
    V_l^{\hat{\pi}, \nu^*(\hat{\pi})} \ge V_l^{\pi^*, \nu^*} - [V_{\varepsilon}^{\hat{\pi}} - V_l^{\hat{\pi}, \nu^*(\hat{\pi})}] - \varepsilon  \ge V_l^{\pi^*, \nu^*} - \text{gap}_\varepsilon - \varepsilon,
    \$
    where the equality uses the definition of $\text{gap}_\varepsilon$ in \eqref{eq:def:gap:vareps}. 
    as desired. Meanwhile, by the facts that $\hat{\nu} \in \hat{\text{BR}}_{3\varepsilon/4}(\hat{\pi})$ and $\hat{\text{BR}}_{3\varepsilon/4}(\hat{\pi}) \subseteq \text{BR}_\varepsilon(\hat{\pi})$, we have 
    \$
    V_{f}^{\hat{\pi}, \hat{\nu}} \ge V_{f}^{\hat{\pi}, \nu^*(\hat{\pi})} - \varepsilon.
    \$
    Therefore, we conclude the proof of Theorem \ref{thm:stackelberg}.
\end{proof}

\section{Proof of Theorem \ref{thm:offline}} \label{appendix:pf:offline}
To facilitate our analysis, we first define the prediction error 
\# \label{eq:7000}
\delta_h = r_{l, h} + \hat{Q}_h - \PP_h \hat{V}_h,
\#
for any $h \in [H]$. Then we show the proof of Theorem \ref{thm:offline}.

\begin{proof}[Proof of Theorem \ref{thm:offline}]

    Similar to Lemma \ref{lemma:best:response}, we have the following lemma.
    \begin{lemma} \label{lemma:best:response2}
    It holds that $\hat{\nu} \in \text{BR}(\hat{\pi})$. Here $\text{BR}(\cdot)$ is defined in \eqref{eq:def:best:response:myopic}.
    \end{lemma}

    \begin{proof}
    This is implied by the definitions of $(\hat{\pi}, \hat{\nu})$ and the assumption that the followers are myopic.
    \end{proof}

    Recall that the definition of optimality gap defined in \eqref{eq:def:optimality:gap} takes the following form
    \#  \label{eq:7001}
	\text{SubOpt}(\hat{\pi}, \hat{\nu}, x) = V_{l, 1}^{\pi^*, \nu^*}(x) - V_{l, 1}^{\hat{\pi}, \hat{\nu}}(x).
	\#
   We decompose this expression by the following lemma.
    \begin{lemma} \label{lemma:value:diff}
    For the $\hat{V}_1$ defined in Line \ref{line:v:l3} of Algorithm \ref{alg3} and any $(\pi, \nu)$, it holds that 
    \$
    V_{l, 1}^{\pi, \nu}(x) - \hat{V}_1(x) &= \EE_{\pi, \nu} \biggl[ \sum_{h = 1}^H \la \hat{Q}_{h}(x_h, \cdot, \cdot), \pi_h(\cdot \,|\, x_h) \times \nu_h(\cdot \,|\, x_h) - \hat{\pi}_h(\cdot \,|\, x_h) \times \hat{\nu}_h(\cdot \,|\, x_h) \ra \biggr]  \\
    & \qquad  + \EE_{\pi, \nu} \biggl[ \sum_{h = 1}^H \delta_{h}(x_h, a_h, b_h) \biggr].
    \$
    \end{lemma}

    \begin{proof}
    This proof is the same as the proof of \eqref{eq:61007}, and we omit it to avoid repetition.
    \end{proof}
    Applying Lemma \ref{lemma:value:diff} with $(\pi, \nu) = ({\pi}^*, {\nu}^*)$, we have 
    \# \label{eq:7002}
    V_{l, 1}^{\pi^*, \nu^*}(x) - \hat{V}_1(x) &= \EE_{\pi^*, \nu^*} \biggl[ \sum_{h = 1}^H \la \hat{Q}_{h}(x_h, \cdot, \cdot), \pi_h^*(\cdot \,|\, x_h) \times \nu_h^*(\cdot \,|\, x_h) -  \hat{\pi}_h(\cdot \,|\, x_h) \times \hat{\nu}_h(\cdot \,|\, x_h) \ra \biggr]  \notag \\
    & \qquad  + \EE_{\pi^*, \nu^*} \biggl[ \sum_{h = 1}^H \delta_{h}(x_h, a_h, b_h) \biggr] .
    \#
    Similarly, applying Lemma \ref{lemma:value:diff} with $(\pi, \nu) = (\hat{\pi}, \hat{\nu})$ gives that 
    \# \label{eq:7003}
    \hat{V}_1(x) - V_{l, 1}^{\hat{\pi}, \hat{\nu}}(x) = - \EE_{\hat{\pi}, \hat{\nu}} \biggl[ \sum_{h = 1}^H \delta_{h}(x_h, a_h, b_h) \biggr].
    \#
    Combining \eqref{eq:7002} and \eqref{eq:7003}, we obtain
    \# \label{eq:7004}
    V_{l, 1}^{\pi^*, \nu^*}(x) - V_{l, 1}^{\hat{\pi}, \hat{\nu}}(x) &= \EE_{\pi^*, \nu^*} \biggl[ \sum_{h = 1}^H \la \hat{Q}_{h}(x_h, \cdot, \cdot), \pi_h^*(\cdot \,|\, x_h) \times \nu_h^*(\cdot \,|\, x_h) -  \hat{\pi}_h(\cdot \,|\, x_h) \times \hat{\nu}_h(\cdot \,|\, x_h) \ra \biggr] \notag\\
    & \qquad  + \EE_{\pi^*, \nu^*} \biggl[ \sum_{h = 1}^H \delta_{h}(x_h, a_h, b_h) \biggr] - \EE_{\hat{\pi}, \hat{\nu}} \biggl[ \sum_{h = 1}^H \delta_{h}(x_h, a_h, b_h) \biggr].
    \#
    As stated in \S \ref{appendix:pf:thm:multi}, these two terms characterize the optimization error and the statistical error, respectively. Similar to Lemmas \ref{lemma:myopic2} and \ref{lemma:ucb2}, we introduce the following two lemmas to analyze these two errors. 
    \begin{lemma} \label{lemma:myopic3}
    It holds that 
    \$
    \EE_{\pi^*, \nu^*} \biggl[ \sum_{h = 1}^H \la \hat{Q}_{h}(x_h, \cdot, \cdot), \pi_h^*(\cdot \,|\, x_h) \times \nu_h^*(\cdot \,|\, x_h) -  \hat{\pi}_h(\cdot \,|\, x_h) \times \hat{\nu}_h(\cdot \,|\, x_h) \ra \biggr] \le \epsilon H.
    \$
    \end{lemma}

    \begin{proof} 
    This proof is similar to the proof of Lemma \ref{lemma:myopic2}, and we omit it to avoid repetition.
    \end{proof}

    \begin{lemma} \label{lemma:ucb3}
    It holds with probability at least $1 - p/2$ that
    \$
    0 \le \delta_h(x, a, b) \le 2 \Gamma_h(x, a, b)
    \$
    for any $h \in [H]$ and $(x, a, b) \in \cS \times \cA_l \times \cA_f$.
    \end{lemma}

    \begin{proof}
    See \S \ref{appendix:pf:ucb3} for a detailed proof.
    \end{proof}

    Combining \eqref{eq:7004} and Lemmas \ref{lemma:myopic3} and \ref{lemma:ucb3}, we further obtain that 
    \#
    V_{l, 1}^{\pi^*, \nu^*}(x) - V_{l, 1}^{\hat{\pi}, \hat{\nu}}(x) &\le \epsilon H + 2 \EE_{\pi^*, \nu^*, x} \biggl[ \sum_{h = 1}^H \Gamma_{h}(x_h, a_h, b_h) \biggr] \notag\\
    & \le 3\beta' \sum_{h = 1}^H \EE_{\pi^*, \nu^*, x} \bigl[  \bigl(\phi(s_h, a_h, b_h)^\top(\Lambda_h)^{-1}\phi(s_h, a_h, b_h)\bigr)^{1/2}  \bigr] ,
    \#
    where the last inequality is obtained by the definition of $\Gamma_h$ in Line \ref{line:def:bonus2} of Algorithm \ref{alg3} and the fact that $\epsilon = d/KH$. Therefore, we conclude the proof of Theorem \ref{thm:offline}.
\end{proof}

\subsection{Proof of Lemma \ref{lemma:ucb3}} \label{appendix:pf:ucb3}

\begin{proof}[Proof of Lemma \ref{lemma:ucb3}]
    Similar to \eqref{eq:54008}, it holds with probability at least $1 - p/2$ that
    \#  \label{eq:54401}
    |\phi(x, a, b)^\top w_{h} - (\PP_h\hat{V}_{h + 1})(x, a, b)| \le \Gamma_{h}(x, a, b)
    \#
    for any $h \in [H]$. The only exception is that we use Lemma \ref{lemma:self:mormalized:offline} instead of the classical concentration lemma (Lemma \ref{lemma:self:normalized}) for the self-normalized process. We omit the detailed proof.

    By \eqref{eq:54401} and the fact that $\hat{V}_{h+1}(\cdot) \le H - h$, we obtain
    \# \label{eq:54402}
    \phi(x, a, b)^\top w_{h} - \Gamma_{h}(x, a, b)  \le (\PP_h\hat{V}_{h + 1})(x, a, b) \le H - h.
    \#
    Thus, we have $\hat{Q}_{h} \ge \phi^\top w_h - \Gamma_h$, which further implies that 
    \# \label{eq:54403}
    \delta_{h}(x, a, b) &= r_{l, h}(x, a, b) + \PP_h\hat{V}_{h+1}(x, a, b) - \hat{Q}_{h}(x, a, b) \notag\\
    & \le \PP_h\hat{V}_{h+1}(x, a, b) - \phi(x, a, b)^\top w_{h} + \Gamma_{h}(x, a, b)  \notag\\
    &\le 2\Gamma_{h}(x, a, b),
    \#
    where the last inequality uses \eqref{eq:54401}. Meanwhile, it holds that 
    \# \label{eq:54404}
    \delta_h(x, a, b) &= r_{l, h}(x, a, b) + \PP_h\hat{V}_{h+1}(x, a, b) - \hat{Q}_h(x, a, b) \notag\\
    & \ge \PP_h\hat{V}_{h+1}(x, a, b) - \max\{\phi(x, a, b)^\top w_{h} - \Gamma_{h}^k(x, a, b), -(H - h)\} \notag\\
    & = \min\{\PP_hV_{h+1}^k(x, a, b) - \phi(x, a, b)^\top w_{h}^k + \Gamma_{h}^k(x, a, b),  \PP_hV_{h+1}^k(x, a, b) + (H - h)\} \notag\\
    & \ge 0 ,
    \#
    where the last inequality follows from \eqref{eq:54401}. Combining \eqref{eq:54403} and \eqref{eq:54404}, we conclude the proof of Lemma \ref{lemma:ucb3}.
\end{proof}

\section{Proof of Corollary \ref{cor:suff:cover}} \label{appendix:pf:cor:suff:cover}

\begin{proof}[Proof of Corollary \ref{cor:suff:cover}]
    The proof is an extension of Corollary 4.5 in \citet{jin2020pessimism}. For notational simplicity, we define 
    \$
    \Sigma_h(x) = \EE_{\pi^*, \nu^*, x}[\phi(s_h, a_h, b_h) \phi(s_h, a_h, b_h)^\top],
    \$
    for all $x \in \cS$ and $h \in [H]$. With this notation and the Cauchy-Schwarz inequality, we have 
    \# \label{eq:45000}
    \EE_{\pi^*, \nu^*, x}\bigl[\sqrt{\phi(s_h, a_h, b_h)^\top\Lambda_h^{-1}\phi(s_h, a_h. b_h)}\bigr] &= \EE_{\pi^*, \nu^*, x}\bigl[\sqrt{\tr\big(\phi(s_h, a_h, b_h)^\top\Lambda_h^{-1}\phi(s_h, a_h, b_h)\bigr)}\bigr] \notag\\
    & = \EE_{\pi^*, \nu^*, x}\bigl[\sqrt{\tr\big(\phi(s_h, a_h. b_h)\phi(s_h, a_h, b_h)^\top\Lambda_h^{-1}\bigr)}\bigr] \notag\\
    & = \EE_{\pi^*, \nu^*}\bigl[\sqrt{\tr\big(\Sigma_h(x)\Lambda_h^{-1}\bigr)}\bigr] .
    \#
    Plugging \eqref{eq:45000} into Theorem \ref{thm:offline}, together with the assumption that $\Lambda_h \succeq I + c \cdot K \cdot \EE_{\pi^*, \nu^*, x}[\phi(s_h, a_h, b_h)\phi(s_h, a_h, b_h)^\top]$ with probability at least $1 - p/2$ and a union bound argument, we further with probability at least $1 - p$ have 
    \# \label{eq:45001}
    \text{SubOpt}(\hat{\pi}, \hat{\nu}, x) &\le 3 \beta' \sum_{h = 1}^H \EE_{\pi^*, \nu^*}\Bigl[\sqrt{\tr\Big(\Sigma_h(x)\bigl(I + c \cdot K \cdot \Sigma_h(x)\bigr)^{-1}\Bigr)}\Bigr] \notag\\
    & = 3 \beta' \sum_{h = 1}^H \sqrt{\sum_{j = 1}^d \frac{\lambda_{h,j}(x)}{1 + cK\lambda_{h,j}(x)}}
    \#
    for all $x \in \cS$. Here $\{\lambda_{h,j}(x)\}_{j = 1}^d$ are the eigenvalues of $\Sigma_h(x)$. Meanwhile, by Jensen's inequality, we obtain
    \# \label{eq:45002}
    \|\Sigma_h(x)\|_{\text{op}} \le \EE_{\pi^*, \nu^*, x}[\|\phi(s_h, a_h, b_h)\phi(s_h, a_h, b_h)^\top\|_{\text{op}}] \le 1,
    \#
    where the last inequality follows from the fact that $\|\phi(\cdot, \cdot, \cdot)\|_2 \le 1$. Combining \eqref{eq:45001} and \eqref{eq:45002}, it holds with probability at least $1 - p$ that
    \$
    \text{SubOpt}(\hat{\pi}, \hat{\nu}, x) &\le 3 \beta' \sum_{h = 1}^H \sqrt{\sum_{j = 1}^d \frac{1}{1 + cK}} \\
    & \le \bar{C} \cdot d^{3/2}H^2 \sqrt{\log(4dHK/p)/K},
    \$
    where $\bar{C} = 3C/\sqrt{c}$, which concludes the proof of Corollary \ref{cor:suff:cover}.
\end{proof}

\section{Results with Pessimistic Tie-breaking} \label{sec:pess:tie:breaking}

\subsection{Stackelberg-Nash Equilibria in Pessimistic Tie-breaking Setup}

For any leader policy $\pi$, we can define 
\# \label{eq:def:against:pess}
\nu^\dagger(\pi) = \{\nu \in \text{BR}(\pi) \,|\, V_{l, h}^{\pi, \nu}(x) \le V_{l, h}^{\pi, \nu'}(x), \forall x \in \cS, h \in [H], \nu' \in \text{BR}(\pi)\},
\# 
where $\text{BR}(\pi)$ is the best-response set defined in \eqref{eq:def:best:response:myopic}. That is, $\nu^\dagger(\pi)$ is the worst-case response in the set $\text{BR}(\pi)$. Then we define the Stackelberg-Nash equilibria by 
\# \label{eq:def:se:pess}
\text{SNE}_l^\dagger = \{\pi  \,|\, V_{l, h}^{\pi, \nu^*(\pi)}(x) \ge V_{l, h}^{\pi', \nu^\dagger(\pi')}(x), \forall x \in \cS, h \in [H], \pi'  \}.
\#
We point out that finding the Stackelberg-Nash equilibria in the pessimistic tie-breaking setting is harder. Specifically, compared with optimistic tie-breaking setting (cf. \eqref{eq:def:bilevel:opt:problem}), 
we need to solve a more complicated constrained max-min optimization problem:
\$
\max_{\pi} \min_{\nu} V_{l, 1}^{\pi, \nu}(x) \qquad \text{s.t.} \, \nu \in \text{BR}(\pi).
\$
Under this more challenging setting, we focus on the leader-controller linear Markov games setting (Assumption \ref{assumption:leader:controller}). Similar to Theorems \ref{thm:multi:follower} and \ref{thm:offline}, we obtain the following two theorems in the online and offline settings respectively.

\subsection{Main Results for the Online Setting}

\begin{theorem}  \label{thm:pessimistic:online}
    Under Assumptions \ref{assumption:linear}, \ref{assumption:leader:controller}, and \ref{assumption:oracle}, there exists an absolute constant $C > 0$ such that, for any fixed $p \in (0, 1)$, by setting $\beta = C \cdot dH \sqrt{\iota}$ with $\iota = \log(2dT/p)$ in Line \ref{line:def:bonus3} of Algorithm~\ref{alg6} and $\epsilon = \frac{1}{KH}$ in Algorithm \ref{alg33}, then have $\nu^k = \nu^\dagger(\pi^k)$ for any $k \in [K]$. Meanwhile, with probability at least $1 - p$, the regret incurred by Algorithm \ref{alg6} satisfies that 
    \$
    \text{Regret}(K) = \sum_{k = 1}^K V_{l, 1}^{\pi^*, \nu^*}(x_1^k) - V_{l, 1}^{\pi^k, \nu^k}(x_1^k) 
	\le \cO(\sqrt{d^3H^3T\iota^2}).
    \$
\end{theorem}

\begin{proof}
    See \S \ref{appendix:pf:thm:pess:online} for a detailed proof.
\end{proof}

\noindent{\bf Misspecification.}
When the transitions do not ideally satisfy the leader-controller assumption, we can potentially consider cases that transitions satisfy, for instance, $|\cP_h(\cdot \,|\, x, a, b) - \cP_h(\cdot \,|\, x, a)\|_{\infty} \le \varrho$ for any $(h, x, a, b) \in [H] \times \cS \times \cA_l \times \cA_f$, Here $\varrho$ is the misspecification error. We can still follow the above method to tackle the misspecified cases. However, because of the misspecification error cumulated during $T$ steps, an extra term $\cO(\varrho T)$ will appear in the final result. In particular, When $\varrho$ is small, that is the Markov games have approximately leader-controller transitions,  the extra term $\cO(\varrho T)$ should be small, which further indicates that we can find SNEs efficiently in some misspecified general-sum Markov games.

\begin{algorithm}[H] 
	\caption{Optimistic Value Iteration to Find Stackelberg-Nash Equilibria (pessimistic tie-breaking version)}
	\begin{algorithmic}[1] \label{alg6}
    	\STATE Initialize $V_{l, H+1}(\cdot) = V_{f, H+1}(\cdot) = 0$.
        \FOR{$k = 1, 2, \cdots, K$}
        \STATE Receive initial state $x_1^k$.
    	\FOR{step $h = H, H-1, \cdots, 1$}
    	\STATE $\Lambda_{h}^k \leftarrow \sum_{\tau = 1}^{k - 1} \phi(x_h^\tau, a_h^\tau) \phi(x_h^\tau, a_h^\tau)^\top + I$. \label{line:def:lambda3}
		\STATE $w_{h}^k \leftarrow (\Lambda_{h}^k)^{-1}\sum_{\tau = 1}^{k - 1}\phi(x_h^\tau, a_h^\tau) \cdot V_{h+1}^k(x_{h+1}^\tau)$. \label{line:def:w3}
		\STATE $\Gamma_{h}^k(\cdot, \cdot, \cdot) \leftarrow \beta \cdot (\phi(\cdot, \cdot)^\top(\Lambda_h^k)^{-1}\phi(\cdot, \cdot))^{1/2}$. \label{line:def:bonus3}
		\STATE $Q_{h}^k(\cdot, \cdot, \cdot) \leftarrow r_{l, h}(\cdot, \cdot, \cdot) + \Pi_{H - h}\{\phi(\cdot, \cdot)^\top w_{h}^k + \Gamma_{h}^k(\cdot, \cdot)\}$. \label{line:def:q3}
		\STATE  $(\pi_h^k(\cdot\,|\,x), \{\nu_{f_i,h}^k(\cdot\,|\, x)\}_{i \in [N]}) \leftarrow \epsilon$-${\text{SNE}}(Q_{h}^k(x, \cdot, \cdot), \{r_{f_i, h}(x, \cdot, \cdot)\}_{i \in [N]})$, $\, \forall x$. (Alg. \ref{alg33})
		\STATE $V_{h}^k(x) \leftarrow \EE_{a \sim \pi_h^k(\cdot\,|\,x), b_1 \sim \nu_{f_1, h}^k(\cdot\,|\,x), \cdots, b_N \sim \nu_{f_N, h}^k(\cdot\,|\,x)}Q_{h}^k(x, a, b_1, \cdots, b_N)$, $\, \forall x$. \label{line:v:l3}
    	\ENDFOR
        \FOR{$h = 1, 2, \cdot, H$}
        \STATE Sample $a_h^k \sim \pi_h^k(\cdot\,|\,x_h^k), b_{1, h}^k \sim \nu_{f_1, h}^k(\cdot\,|\,x_h^k), \cdots, b_{N, h}^k \sim \nu_{f_N, h}^k(\cdot\,|\,x_h^k)$.
        \STATE Leader takes action $a_h^k$; Followers take actions $b_h^k = \{b_{i, h}^k\}_{i \in [N]}$.    
        \STATE Observe next state $x_{h+1}^k$.
        \ENDFOR
        \ENDFOR
	\end{algorithmic}
\end{algorithm}

\begin{algorithm}[H] 
	\caption{$\epsilon$-SNE (pessimistic tie-breaking version)}
	\begin{algorithmic}[1] \label{alg33}
	\STATE {\bf Input:} $Q_h^k, x$, and parameter $\epsilon$.
	\STATE Select $\tilde{Q}$ from $\cQ_{h, \epsilon}^k$ satisfying $\|\tilde{Q} - Q_h^k\|_{\infty} \le \epsilon$. 
	\STATE For the input state $x$, let $(\pi_h^k(\cdot\,|\,x), \{\nu_{f_i,h}^k(\cdot\,|\, x)\}_{i \in [N]})$ be the Stackelberg-Nash equilibrium for the matrix game with payoff matrices $(\tilde{Q}(x, \cdot, \cdot), \{r_{f_i, h}(x, \cdot, \cdot)\}_{i \in [N]})$ in the pessimistic tie-breaking setting. \label{line:matrix:game3}
	\STATE {\bf Output:} $(\pi_h^k(\cdot\,|\,x), \{\nu_{f_i,h}^k(\cdot\,|\, x)\}_{i \in [N]})$. 
	\end{algorithmic} 
\end{algorithm}

\subsection{Main Results for the Offline Setting}

\begin{theorem}  \label{thm:pessimistic:offline} 
    Under Assumptions \ref{assumption:linear}, \ref{assumption:leader:controller}, \ref{assumption:oracle}, and~\ref{assumption:compliance:data}, there exists an absolute constant $C > 0$ such that, for any fixed $p \in (0, 1)$, by setting $\beta' = C \cdot dH\sqrt{\log(2dHK/p)}$ in Line \ref{line:def:bonus4} of Algorithm~\ref{alg4} and $\epsilon = \frac{d}{KH}$ in Algorithm \ref{alg33}, then we have $\hat{\nu} = \nu^\dagger(\hat{\pi})$. Meanwhile, with probability at least $1 - p$, we have 
    \$
    \text{SubOpt}(\hat{\pi}, \hat{\nu}, x) = V_{l, 1}^{\pi^*, \nu^*}(x) - V_{l, 1}^{\hat{\pi}, \hat{\nu}}(x) \le 3\beta' \sum_{h = 1}^H \EE_{\pi^*, x} \bigl[  \bigl(\phi(s_h, a_h)^\top(\Lambda_h)^{-1}\phi(s_h, a_h)\bigr)^{1/2}  \bigr],
    \$
	where $\EE_{\pi^*, x}$ is taken with respect to the trajectory incurred by $\pi^*$ in the underlying leader-controller Markov game when initializing the progress at $x$. Here $\Lambda_h$ is defined in Line \ref{line:def:lambda4} of Algorithm \ref{alg4}.
\end{theorem}

\begin{proof}
	Combining the proofs of Theorems \ref{thm:offline} and \ref{thm:pessimistic:online}, we can conclude the proof of Theorem~\ref{thm:pessimistic:offline}. To avoid repetition, we omit the detailed proof here. 
\end{proof}

\vskip 4pt
\noindent{\bf Optimality of the Bound:}
	Assuming dummy followers---that is, the actions taken by the followers won't affect the reward functions and transition kernels---the Markov game reduces to a linear MDP \citep{jin2020provably}. Together with the information-theoretic lower bound $\Omega(\sum_{h=1}^H\EE_{\pi^*, x} [ (\phi(s_h, a_h)^\top(\Lambda_h)^{-1}\phi(s_h, a_h))^{1/2}])$ established in \citet{jin2020pessimism} for linear MDPs, we immediately obtain the same lower bound for our setting. In particular, our upper bound established in Theorem \ref{thm:pessimistic:offline} matches this lower bound up to $\beta'$ and absolute constants and thus implies that our algorithm is nearly minimax optimal.  

\begin{algorithm}[H]
	\caption{Pessimistic Value Iteration to Find Stackelberg-Nash Equilibria (pessimistic tie-breaking version)}
	\begin{algorithmic}[1] \label{alg4}
		\STATE {\bf Input:} $\cD = \{x_h^\tau, a_h^\tau, b_h^\tau = \{b_{i, h}^\tau\}_{i \in [N]}\}_{\tau, h = 1}^{K, H}$ and reward functions $\{r_l, r_f = \{r_{f_i}\}_{i \in [N]}\}$.
    	\STATE Initialize $\hat{V}_{H+1}(\cdot) = 0$.
    	\FOR{step $h = H, H-1, \cdots, 1$}
    	\STATE $\Lambda_{h} \leftarrow \sum_{\tau = 1}^{K} \phi(x_h^\tau, a_h^\tau) \phi(x_h^\tau, a_h^\tau)^\top + I$. \label{line:def:lambda4} 
		\STATE $w_{h} \leftarrow (\Lambda_{h})^{-1}\sum_{\tau = 1}^{K}\phi(x_h^\tau, a_h^\tau) \cdot \hat{V}_{h+1}(x_{h+1}^\tau)$. \label{line:def:w4}
		\STATE $\Gamma_{h}(\cdot, \cdot) \leftarrow \beta' \cdot (\phi(\cdot, \cdot)^\top(\Lambda_h)^{-1}\phi(\cdot, \cdot))^{1/2}$. \label{line:def:bonus4}
		\STATE $\hat{Q}_{h}(\cdot, \cdot, \cdot) \leftarrow r_{l, h}(\cdot, \cdot, \cdot) + \Pi_{H - h}\{\phi(\cdot, \cdot)^\top w_{h} - \Gamma_{h}(\cdot, \cdot)\}$. \label{line:def:q4}
		\STATE  $(\hat{\pi}_h(\cdot\,|\,x), \{\hat{\nu}_{f_i,h}(\cdot\,|\, x)\}_{i \in [N]}) \leftarrow \epsilon$-SNE$(\hat{Q}_{h}(x, \cdot, \cdot), \{r_{f_i, h}(x, \cdot, \cdot)\}_{i \in [N]})$, $\, \forall x$. (Alg. \ref{alg33}) \label{line:sne4}
		\STATE $\hat{V}_{h}(x) \leftarrow \EE_{a \sim \hat{\pi}_h(\cdot\,|\,x), b_1 \sim \hat{\nu}_{f_1, h}(\cdot\,|\,x), \cdots, b_N \sim \hat{\nu}_{f_N, h}(\cdot\,|\,x)}\hat{Q}_{h}(x, a, b_1, \cdots, b_N)$, $\, \forall x$. \label{line:v:l4}
    	\ENDFOR
		\STATE {\bf Output:} $(\hat{\pi} = \{\hat{\pi}_h\}_{h = 1}^H, \hat{\nu} = \{\hat{\nu}_{f_i} = \{\nu_{f_i, h}\}_{h = 1}^H\}_{i = 1}^N)$.
	\end{algorithmic}
\end{algorithm} 

\subsection{Proof of Theorem \ref{thm:pessimistic:online}} \label{appendix:pf:thm:pess:online}

\begin{proof}[Proof of Theorem \ref{thm:pessimistic:online}]
	For leader-controller Markov games, we have a stronger version of Lemma \ref{lemma:best:response}.

	\begin{lemma} \label{lemma:best:response:pess}
		For any $k \in [K]$, we have $\nu^k = \nu^\dagger(\pi^k)$. Here $\nu^\dagger(\cdot)$ is defined in \eqref{eq:def:against:pess}.
	\end{lemma}
	
	\begin{proof}
        Fix $k \in [K]$, by the definition of the best response in \eqref{eq:def:best:response2}, we have 
        \# \label{eq:66000}
        \text{BR}(\pi^k) &= \{\nu = \{\nu_{f_i}\}_{i \in [N]} \,|\, \nu \text{ is the NE of the followers given the leader policy } \pi^k\} \notag\\
        & = \{\nu = \{\nu_{f_i}\}_{i \in [N]} \,|\, \nu \text{ is the NE of }  \{V_{f_i, h}^{\pi^k, \nu}(x)\}_{i \in [N]}, \, \forall h \in [H] \text{ and } x \in \cS \} \notag\\
        & = \{\nu = \{\nu_{f_i}\}_{i \in [N]} \,|\, \nu \text{ is the NE of }  \{r_{f_i, h}^{\pi^k, \nu}(x)\}_{i \in [N]}, \, \forall h \in [H] \text{ and } x \in \cS \},
        \# 
        where $r_{f_i, h}^{\pi^k, \nu}(x) = \la r_{f_i, h}(x, \cdot, \cdot, \cdots, \cdot), \pi^k_h(\cdot \,|\, x) \times \nu_{f_1, h}(\cdot \,|\, x) \times \cdots \times \nu_{f_N, h}(\cdot \,|\, x)\ra_{\cA_l \times \cA_{f}}$. Here the last inequality uses Bellman equality \eqref{eq:bellman} and the leader-controller assumption. Moreover, by the definition of $\nu^\dagger(\pi^k)$ defined in \eqref{eq:def:against2}, we have that 
        \# \label{eq:66001}
        \nu_h^\dagger(\pi^k) = \{\nu_{f_i, h}^\dagger(\pi^k)\}_{i \in [N]} \in &\argmin_{\nu \in \text{BR}(\pi^k)}  V_{l, h}^{\pi^k, \nu}(x) = \argmin_{\nu \in \text{BR}(\pi^k)}  r_{l, h}^{\pi^k, \nu}(x) ,
        \#
        where $r_{l, h}^{\pi^k, \nu}(x) = \la r_{l, h}(x, \cdot, \cdot, \cdots, \cdot), \pi^k_h(\cdot \,|\, x) \times \nu_{f_1, h}(\cdot \,|\, x) \times \cdots \times \nu_{f_N, h}(\cdot \,|\, x)\ra_{\cA_l \times \cA_{f}}$. Here the last equality uses the single-controller assumption.
    
        Recall that, in the subroutine $\epsilon$-SNE (Algorithm \ref{alg22}), we pick the function $\tilde{Q} \in \cQ_{h, \epsilon}^k$ such that $\|Q_h^k - \tilde{Q}\|_{\infty} \le \epsilon$ and solve the matrix game defined in \eqref{eq:def:matrix:game}. Here $\cQ_{h, \epsilon}^k$ is the class of functions $Q: \cS \times \cA_l \times \cA_f \rightarrow \RR$ that takes form
        \# \label{eq:66002}
        Q(\cdot, \cdot, \cdot) = r_{l, h}(\cdot, \cdot, \cdot) + \Pi_{H - h}\big\{\phi(\cdot, \cdot)^\top w + \beta \cdot \big(\phi(\cdot, \cdot)^\top \Lambda^{-1} \phi(\cdot, \cdot)\big)^{1/2}\big\}, 
        \#
        where $\|w\|_2 \le H\sqrt{dk}$ and $\lambda_{\min}(\Lambda) \geq 1$. Thus, given the leader policy $\pi^k$, the best response of the followers for the matrix game defined in \eqref{eq:def:matrix:game} takes the form 
        \# \label{eq:66003}
         \text{BR}'(\pi^k) 
        & = \{\nu \,|\, \nu \text{ is the NE of } \{\la r_{f_i, h}(x, \cdot, \cdot), \pi_h^k(\cdot \,|\, x) \times \nu_h(\cdot \,|\, x) \ra\}_{i \in [N]}, \forall h \in [H] \text{ and } x \in \cS\} \notag\\ 
        & = \text{BR}(\pi^k)
        \#
        where $\la r_{f_i, h}(x, \cdot, \cdot), \pi_h^k(\cdot \,|\, x) \times \nu_h(\cdot \,|\, x) \ra$ is the shorthand of $\la r_{f_i, h}(x, \cdot, \cdot, \cdots, \cdot), \pi_h^k(\cdot \,|\, x) \times \nu_{f_1, h}(\cdot \,|\, x) \times \cdots \times  \nu_{f_N, h}(\cdot \,|\, x) \ra_{\cA_l \times \cA_{f}}$. Here the last equality uses \eqref{eq:66000}. Similarly, by the definition of ${\cQ}_{h, \epsilon}^k$ in \eqref{eq:66002}, we can obtain that 
        \# \label{eq:66004}
        &\argmin_{\nu_h} \la \tilde{Q}(x, \cdot, \cdot), \pi_h^k(\cdot \,|\, x) \times \nu_h(\cdot \,|\, x) \ra  = \argmin_{\nu_h} \la r_{l, h}(x, \cdot, \cdot), \pi_h^k(\cdot \,|\, x) \times \nu_h(\cdot \,|\, x) \ra ,
        \#
        where $\la r_{l, h}(x, \cdot, \cdot), \pi_h^k(\cdot \,|\, x) \times \nu_h(\cdot \,|\, x) \ra$ is the abbreviation of $\la r_{f_i, h}(x, \cdot, \cdot, \cdots, \cdot), \pi_h^k(\cdot \,|\, x) \times \nu_{f_1, h}(\cdot \,|\, x) \times \cdots \times  \nu_{f_N, h}(\cdot \,|\, x) \ra_{\cA_l \times \cA_{f}}$
        Together with \eqref{eq:66001} and \eqref{eq:66003}, we have that,  for the matrix game with payoff matrices $(\tilde{Q}(x_h^k, \cdot, \cdot), \{r_{f_i, h}^k(x_h^k, \cdot, \cdot)\}_{i \in [N]})$, the policy $\nu_h^k(\cdot \,|\, x_h^k) = \{\nu_{f_i, h}^k(\cdot \,|\, x_h^k)\}_{i \in [N]}$ is also the best response of $\pi_h^k(\cdot \,|\, x_h^k)$ and breaks ties against favor of the leader. Therefore, we have $\nu^k = \nu^\dagger(\pi^k)$ for any $k \in [K]$, which concludes the proof of Lemma~\ref{lemma:best:response:pess}. 
    \end{proof}
    Then we only need to bound the quantity $\sum_{k = 1}^K\sum_{k = 1}^K V_{l, 1}^{\pi^*, \nu^*}(x_1^k) - V_{l, 1}^{\pi^k, \nu^k}(x_1^k)$. By Lemma~\ref{lemma:decomposition2}, we have 
	\$
    \text{Regret}(K) & =   \underbrace{\sum_{k = 1}^K\sum_{h=1}^H\EE_{\pi^*, \nu^*}[\la Q_{h}^k(x_h^k, \cdot, \cdot), \pi_h^*(\cdot \,|\, x_h^k) \times \nu_h^*(\cdot \,|\, x_h^k) -  \pi_h^k(\cdot \,|\, x_h^k) \times \nu_h^k(\cdot \,|\, x_h^k) \ra]}_{{(l.1):} \text{ Computational Error}} \\
    & \quad + \underbrace{\sum_{k = 1}^K\sum_{h=1}^H \bigl(\EE_{\pi^*, \nu^*}[\delta_{h}^k(x_h, a_h, b_{h})] - \delta_{h}^k(x_h^k, a_h^k, b_{h}^k) \bigr)}_{(l.2): \text{ Statistical Error}} + \underbrace{\sum_{k = 1}^K\sum_{h = 1}^H (\zeta_{k, h}^1 + \zeta_{k, h}^2)}_{(l.3): \text{ Randomness}}, 
    \$
    where $\la Q_{h}^k(x_h^k, \cdot, \cdot), \pi_h^*(\cdot \,|\, x_h^k) \times \nu_h^*(\cdot \,|\, x_h^k) -  \pi_h^k(\cdot \,|\, x_h^k) \times \nu_h^k(\cdot \,|\, x_h^k) \ra = \la Q_{h}^k(x_h^k, \cdot, \cdot, \cdots, \cdot), \pi_h^*(\cdot \,|\, x_h^k) \times \nu_{f_1, h}^*(\cdot \,|\, x_h^k) \times \cdots \nu_{f_N, h}^*(\cdot \,|\, x_h^k) -  \pi_h^k(\cdot \,|\, x_h^k) \times \nu_{f_1, h}^k(\cdot \,|\, x_h^k) \times \cdots \nu_{f_N, h}^k(\cdot \,|\, x_h^k) \ra_{\cA_l \times \cA_{f}}$.  
	
	By the same argument of Lemma \ref{lemma:best:response:pess}, we have that, for the matrix game with payoff matrices $(\tilde{Q}(x_h^k, \cdot, \cdot), \{r_{f_i, h}^k(x_h^k, \cdot, \cdot)\}_{i \in [N]})$, $\nu_h^*(\cdot \,|\, x_h^k)$ belongs to the best response set of $\pi_h^*(\cdot \,|\, x_h^k)$ and breaks ties against favor of the leader.
    Recall that $(\pi_h^k(\cdot \,|\, x_h^k), \nu_h^k(\cdot \,|\, x_h^k) = \{\nu_{f_i, h}^k(\cdot \,|\, x_h^k)\}_{i \in [N]})$ is the Stackelberg-Nash equilibrium of the matrix game with payoff matrices $(\tilde{Q}(x_h^k, \cdot, \cdot, \cdot), \{r_{f_i, h}^k(x_h^k, \cdot, \cdot, \cdot)\}_{i \in [N]})$ in the pessimistic tie-breaking setting, which implies that $\pi_h^k(\cdot \,|\, x_h^k)$ is the ``worst response to the best response,'' which further implies that  
    \# \label{eq:66005}
    \la \tilde{Q}(x_h^k, \cdot, \cdot), \pi_h^*(\cdot \,|\, x_h^k) \times \nu_h^*(\cdot \,|\, x_h^k) -  \pi_h^k(\cdot \,|\, x_h^k) \times \nu_h^k(\cdot \,|\, x_h^k) \ra \le 0
    \#
    for any $(k, h) \in [K] \times [H]$. Thus, for any $(k, h) \in [K] \times [H]$, we have 
    \# \label{eq:66006}
    &\la {Q}_h^k(x_h^k, \cdot, \cdot), \pi_h^*(\cdot \,|\, x_h^k) \times \nu_h^*(\cdot \,|\, x_h^k) -  \pi_h^k(\cdot \,|\, x_h^k) \times \nu_h^k(\cdot \,|\, x_h^k) \ra \notag\\
    & \qquad = \la \tilde{Q}(x_h^k, \cdot, \cdot), \pi_h^*(\cdot \,|\, x_h^k) \times \nu_h^*(\cdot \,|\, x_h^k) -  \pi_h^k(\cdot \,|\, x_h^k) \times \nu_h^k(\cdot \,|\, x_h^k) \ra \notag\\
    & \qquad \qquad + \la Q_h^k(x_h^k, \cdot, \cdot) - \tilde{Q}(x_h^k, \cdot, \cdot), \pi_h^*(\cdot \,|\, x_h^k) \times \nu_h^*(\cdot \,|\, x_h^k) -  \pi_h^k(\cdot \,|\, x_h^k) \times \nu_h^k(\cdot \,|\, x_h^k) \ra \notag\\
    & \qquad \le \epsilon,
    \#
    where the last inequality uses~\eqref{eq:63005} and the fact that $\|Q_h^k - \tilde{Q}\|_{\infty} \le \epsilon$. By taking summation over $(k, h) \in [K] \times [H]$, we bound the computational error as desired. Moreover, we can characterize statistical error by Lemmas \ref{lemma:ucb2} and \ref{lemma:elliptical2}. The remaining randomness term can be bounded by Lemma \ref{lemma:martingale2}. Putting these together, we have $\text{Regret}(K) \le \cO(\sqrt{d^3H^3T\iota^2})$, which concludes the proof of Theorem \ref{thm:pessimistic:online}.
\end{proof}

\section{Supporting Lemmas}

\begin{lemma}[Elliptical Potential Lemma \citep{dani2008stochastic, abbasi2011improved, jin2020provably,cai2020provably}] \label{lemma:abbasi}
    Let $\{ \phi_t \}_{t=1}^\infty$ be an $\RR^d$-valued sequence. Meanwhile, let $\Lambda_0\in\RR^{d\times d}$ be a positive-definite matrix and $\Lambda_t=\Lambda_0 + \sum_{j=1}^{t-1} \phi_j\phi_j^\top$. It holds for any $t\in \ZZ_+$ that
    \$
     \sum_{j=1}^t \min\{1, \|\phi_j\|^2_{\Lambda^{-1}_{j}} \} \le
    2\log\biggl( \frac{\det(\Lambda_{t+1})}{\det(\Lambda_1)} \biggr).
    \$
    \end{lemma}
    \begin{proof}
    See Lemma 11 of \cite{abbasi2011improved} for a detailed proof.
    \end{proof}

    \begin{lemma}[Concentration of Self-Normalized Process \citep{abbasi2011improved}]\label{lemma:self:normalized}
        Let $\{ \tilde{\cF}_t \}^\infty_{t=0}$ be a filtration and $\{\eta_t\}^\infty_{t=1}$ be an $\RR$-valued stochastic process such that $\eta_t$ is $\tilde{\cF}_t$-measurable for any $t\geq 0$. 
        We also assume that, for any $t\geq 0$, conditioning on $\tilde \cF_t$,  $\eta_t$ is a zero-mean and  $\sigma$-sub-Gaussian random variable, that is,
        \#\label{eq:define_subgaussian}
        \EE[\eta_t \,|\,\tilde{\cF}_t] = 0, \qquad \EE[e^{\lambda \eta_t} \,|\, \tilde{\cF}_t ] \le e^{\lambda^2\sigma^2/2},
        \#
        for any $\lambda\in\RR$. Let $\{X_t\}^\infty_{t=1}$ be an $\RR^d$-valued stochastic process such that $X_t$ is $\tilde{\cF}_t$-measurable for any $t\geq 0$. Also, let  $Y \in \RR^{d\times d}$ be  a deterministic and  positive-definite matrix. For any $t\ge0$, we  define
        \$
        \overline{Y}_t = Y + \sum_{s=1}^t X_s X_s^\top, \quad S_t=\sum_{s=1}^t \eta_s\cdot X_s.
        \$
        For any $\delta>0$ and $t\ge0$., it holds with probability at least $1-\delta$ that 
        \$
        \| S_t \|^2_{\overline{Y}^{-1}_t} \le 2\sigma^2\cdot \log\biggl( \frac{\det(\overline{Y}_t)^{1/2}\det(Y)^{-1/2}}{\delta} \biggr) .
        \$
        \end{lemma}
        \begin{proof}
        See Theorem 1 of \cite{abbasi2011improved} for a detailed proof.
        \end{proof}

    \begin{lemma} \label{lemma:self:mormalized:offline}
       For any fixed $h \in [H]$, let $V: \cS \rightarrow [0, H]$ be any fixed value function. Under Assumption \ref{assumption:compliance:data}, for any fixed $\delta > 0$, we have
       {\small
       \$
       &P_{\cD} \biggl(\Big\|\sum_{k = 1}^K \phi(x_h^\tau, a_h^\tau, b_h^\tau) \cdot \bigl(V(x_{h + 1}^\tau) - \PP_h V(x_h^\tau, a_h^\tau, b_h^\tau) \bigr)\Big\|_{\Lambda_h^{-1}} > H^2 \cdot \bigl(2\log(1/\delta) + d \cdot \log(1 + K) \bigr)\biggr) \le \delta.
       \$}
    \end{lemma}

    \begin{proof}
       See Lemma B.2 of \citet{jin2020pessimism} for a detailed proof.
    \end{proof}

        \begin{lemma}[Covering] \label{lemma:covering}
            Let $\cQ_h$ be the class of value functions $Q: \cS \times \cA_l \times \cA_f \rightarrow \RR$ that takes the form 
            \$
            Q(\cdot, \cdot, \cdot) = r_{l, h}(\cdot, \cdot, \cdot) + \Pi_{H - h}\{(\phi(\cdot, \cdot, \cdot)^\top w + \beta \cdot \bigl( \phi(\cdot, \cdot, \cdot)^\top \Lambda^{-1} \phi(\cdot, \cdot, \cdot) \bigr)^{1/2} \},
            \$
            which are parameterized by $(w, \Lambda) \in \RR^d \times \RR^{d \times d}$ such that $\|w\| \le L$ and $\lambda_{\text{min}}(\Lambda) \ge \lambda$. We assume that $\beta$ is fixed and satisfy that $\beta \in [0, B]$, and the feature map $\phi: \cS \times \cA \rightarrow \RR^d$ satisfies that $\|\phi(\cdot, \cdot)\|_2 \le 1$. We have that, for any $L, B, \epsilon > 0$, there exists an $\epsilon$-covering of $\cQ_h$ with respect to the $\ell_{\infty}$ norm such that the covering number $\cN_\epsilon$ satisfies
            \$
            \log\cN_{\epsilon} \le d \cdot \log(1 + 4L/\epsilon) + d^2 \cdot \log\bigl(1 + 8B^2\sqrt{d}/(\epsilon^2 \lambda) \bigr).
            \$
        \end{lemma}

        \begin{proof}
            See \citet{jin2020provably} for a detailed proof.
        \end{proof}

\end{appendix}

\end{document}